\documentclass[11pt,a4paper]{article}
\usepackage{authblk}
\usepackage{graphicx}
\usepackage{amsmath}
\usepackage{amssymb}
\usepackage{algorithm}
\usepackage{algorithmic}
\usepackage{amsthm}
\usepackage{epstopdf}
\usepackage{caption}
\usepackage{url}
\usepackage{subcaption}
\usepackage{natbib}
\usepackage{cases}
\newtheorem{theorem}{Theorem}
\newtheorem{lemma}{Lemma}
\newtheorem{proposition}{Proposition}
\newtheorem{corollary}{Corollary}

\theoremstyle{definition}

\title{\bf Linear Convergence of SDCA in Statistical Estimation }
\date{}
\author[1]{Chao Qu}
\author[2]{Huan Xu}
\affil[1]{Department of Mechanical Engineering, National University of Singapore}
\affil[2]{H. Milton Stewart School of Industrial and Systems Engineering, Georgia Institute of Technology}

\begin{document}
	
	\maketitle
	
\begin{abstract}
	In this paper, we consider stochastic dual coordinate (SDCA) {\em without} strongly convex assumption or convex assumption. We show that SDCA converges linearly under mild conditions termed restricted strong convexity. This covers a wide array of popular statistical models including  Lasso, group Lasso, and logistic regression with $\ell_1$ regularization, corrected Lasso  and linear regression with SCAD regularizer. This significantly improves previous convergence results on SDCA for problems that are not strongly convex. As a by product, we derive a dual free form of SDCA that can handle general regularization term, which is of interest by itself.

\end{abstract}

	\section{Introduction}
	First order methods have again become a central focus of research in optimization and particularly in machine learning in recent years, thanks to its ability to address very large scale empirical risk minimization problems that are ubiquitous in machine learning, a task that is often challenging for other algorithms such interior point methods.  The randomized (dual) coordinate version of the first order method samples one data point and updates the objective function at each time step, which avoids the computations of the full gradient and  pushes the speed to a higher level. Related methods have been implemented in various software packages \cite{vedaldi08vlfeat}. In particular, the randomized dual coordinate method considers the following problem.
	\begin{equation}\label{ERM}\begin{split}
	\min_{w\in \Omega} F(w)&:=\frac{1}{n}\sum_{i=1}^{n} f_i (w)+\lambda g(w) \\
	&= f(w)+\lambda g(w),
	\end{split}
	\end{equation}
	where $f_i(w)$ is a convex loss function of each sample and $g(w)$ is the regularization, $\Omega$ is a convex compact set.
	Instead of directly solving the primal problem it look at the dual problem
	$$ D(\alpha)=\frac{1}{n}\sum_{i=1}^{n}-\psi_i^* (-\alpha_i)-\lambda g^*\big( \frac{1}{\lambda n}\sum_{i=1}^{n}X_i\alpha_i \big),$$
	where it assumes the loss function $f_i(w)$ has the form $\psi_i(X_i^Tw)$.  \citet{shalev2013stochastic} consider this dual form and proved linear convergence of the stochastic dual coordinate ascent method (SDCA) when $g(w)=\|w\|_2^2$. They further extend the result into a general form, which allows the regulerizer $g(w)$
	to be a general strongly convex function \cite{shalev2014accelerated}. \citet{qu2015stochastic} propose the AdaSDCA, an adaptive variant of SDCA, to allow the method to adaptively change the probability distribution over the dual variables through the iterative process. Their experiment results outperform the non-adaptive methods. \citet{dunner2016primal} consider a primal-dual frame work and extend SDCA to the non-strongly convex case with sublinear rate. \citet{allen2015even} further improve the convergece speed using a novel non-uniform sampling that selects each coordinate with a probability proportional to the square root of the smoothness parameter. Other acceleration techniques \cite{qu2014coordinate,qu2015quartz,nesterov2012efficiency},  as well as mini-batch and distributed variants on coordinate method \cite{liu2015asynchronous,zhao2014accelerated,jaggi2014communication,mahajan2014distributed} have been studied in literature. See \citet{wright2015coordinate} for a review on the coordinate method. 
	
	Our goal is to investigate how to extend SDCA to non-convex statistical problems. Non-convex optimization problems attract fast growning attentions due to the rise of numerous applications, notably non-convex M estimators (e.g.,SCAD \cite{fan2001variable}, MCP \cite{zhang2010nearly}),  deep learning \cite{Goodfellow-et-al-2016} and robust regression \citep[e.g., corrrect Lasso in][]{loh2011high}. \citet{shalev2016sdca} proposes a dual free version of SDCA and proves the linear convergence of it, which addresses the case that each individual loss function $f_i(w)$ may be non-convex, but their sum $f(w)$ is strongly convex.  This extends the applicability of SDCA.  
	From a technical perspective, due to non-convexity of $f_i(w)$, the paper avoids explicitly using its dual (hence the name ``dual free''), by introducing pseudo-dual variables.
	
	In this paper, we consider using SDCA to solve M-estimators that are not strongly convex, or even non-convex. We show that under the {\em restricted strong convexity} condition, SDCA converges linearly. This setup includes well known formulations such as   Lasso, Group Lasso, Logistic regression with $\ell_1$ regularization, linear regression with SCAD regularizer \cite{fan2001variable}, Corrected Lasso \cite{loh2011high}, to name a few. This significantly improves upon existing theoretic results~\cite{shalev2013stochastic,shalev2014accelerated,shalev2016sdca} which only established sublinear convergence on convex objectives.
	
	
	To this end, we first adapt  SDCA \cite{shalev2014accelerated} into  a  generalized dual free form in  Algorithm \ref{alg:proximal_SDCA}. This is because to apply SDCA in our setup, we need to introduce a non-convex term and thus demand a dual free analysis.
	We remark that the theory about dual free SDCA established in \citet{shalev2016sdca} {\em does  not} apply to our setup for the following reasons:
	\begin{enumerate}
		\item In \citet{shalev2016sdca}, $F(w)$ needs to be strongly convex, thus does not apply to Lasso or non-convex M-estimators.
		\item \citet{shalev2016sdca} only studies the special case where $g(w)=\frac{1}{2}\|w\|_2^2$, while the M-estimators we consider include non-smooth regularization such as $\ell_1$ or $\ell_{1,2}$ norm.
	\end{enumerate}

	\begin{algorithm}[t]
		\caption{Generalized Dual free SDCA}
		\label{alg:proximal_SDCA}
		\begin{algorithmic}
			\STATE {\bfseries Input:}  Step size $\eta$, number of iterations $T$, and smoothness parameters $L_1,...,L_n$ 
			\STATE {\bfseries Initializition:}   $v^0=\frac{1}{\lambda n} \sum_{i=1}^{n} a_i^0$ for some $a^0=(a_1^0,...,a_n^0)$, $q_i=\frac{L_i+\bar{L}}{2n\bar{L}}$, where $\bar{L}=\frac{1}{n}\sum_{i=1}^{n}L_i$, probability distribution $Q=(q_1,q_2,...,q_n)$ on $\{1,...,n\}$.
			\FOR{$t=1,...,T$}
			\STATE pick $i$ randomly according to $Q$, $\eta_i=\frac{\eta}{q_i n}$; \\
			\STATE update: $a_i^t=a_i^{t-1}-\eta_i \lambda n (\nabla f_i(w^{t-1})+a_i^{t-1})$;\\
			\STATE update: $v^t=v^{t-1}-\eta_i  (\nabla f_i(w^{t-1})+a_i^{t-1})$;\\
			\STATE update:  $w^t=\arg \max_{w\in \Omega} \langle w,v^t\rangle-g(w) $.
			\ENDFOR
		\end{algorithmic}
	\end{algorithm}

	We   illustrate the update rule of  Algorithm \ref{alg:proximal_SDCA} using an example. While our main focus  is for   non-strongly convex problems, we start with a strongly convex example to obtain some intuitions. Suppose $f_i(w)=\frac{1}{2}(y_i-w^T x_i)^2$ and $g(w)=\frac{1}{2}\|w\|_2^2+ \|w\|_1$. It is easy to see $\nabla f_i(w)= x_i(w^T x_i-y_i)$ and $	w_t=\arg\min_{w\in\Omega} \frac{1}{2}\|w-v_t\|_2^2+ \|w\|_1.$
	
	It is then clear that to apply SDCA  $g(w)$ needs to be strongly convex, or otherwise the proximal step $w_t=\arg\max_{w\in \Omega} w^Tv^t-\|w\|_1  $ becomes ill-defined: $w_t$ may be infinity (if $\Omega$ is unbounded) or not unique. This observation motivates the following preprocessing step: if $g(\cdot)$ is not strongly convex, for example Lasso where  $f_i(w)=\frac{1}{2}(y_i-w^T x_i)^2$ and $g(w)=\|w\|_1$, we redefine the formulation by adding a strongly convex term to $g(w)$ and subtract it from $f(w)$.
	More precisely, for $i=1,...,n$, we define $	\phi_i(w)=\frac{n+1}{2n} \|y_i-w^Tx_i\|_2^2,$ $\phi_{n+1}(w)=-\frac{ \tilde{\lambda} (n+1)}{2}\|w\|_2^2,$ $\tilde{g}(w)= \frac{1}{2}\|w\|_2^2+\frac{\lambda}{\tilde{\lambda}}\|w\|_1,$
	and apply Algorithm \ref{alg:proximal_SDCA} on $\frac{1}{n+1} \sum_{i=1}^{n+1} \phi_i(w)+\tilde{\lambda} \tilde{g}(w)$ (which is equivalent to Lasso), where the value of $\tilde{\lambda}$ will be specified later. Our analysis  is thus focused on this alternative representation. The main challenge arises from the  non-convexity of the newly defined $\phi_{n+1}$, which precludes the use of the {\em dual method} as in \cite{shalev2014accelerated}. While in \citet{shalev2016sdca}, a dual free SDCA algorithm is proposed and analyzed, the results do not apply to out setting for reasons discussed above.  
	
	Our contributions are two-fold. 1. We prove   linear convergence  of SDCA for a class of problem that are not strongly convex   or even non-convex, making use of the concept  {\em restricted strong convexity} (RSC). To the best of our knowledge, this is the first work to prove  linear convergence   of SDCA under this setting which includes several statistical model such as Lasso, Group Lasso, logistic regression, linear regression with SCAD regularization, corrected Lasso, to name a few.  2. As a byproduct,  we derive a dual free from SDCA that extends the work of \citet{shalev2016sdca} to account for more general regularization $g(w)$.

	\textbf{Related work.}\,\, \citet{agarwal2010fast} prove  linear convergence  of the {\em batched} gradient and composite gradient method in Lasso, Low rank matrix completion problem  using the RSC condition. \citet{loh2013regularized} extend this to a class of non-convex M-estimators. In spirit, our work can be thought as a stochastic counterpart. Recently, \citet{qu2016linear,qu2017saga} consider SVRG and SAGA in the similar setting as ours, but they look at the primal problem. \citet{shalev2016sdca} considers dual free SDCA, but the analysis does not apply to the case where $F(\cdot)$ is not strongly convex.  Similarly, \citet{allen2015improved}  consider the non-strongly convex setting in SVRG ($\mbox{SVRG}^{++}$), where $F(w)$ is convex, each individual $f_i(w)$ may be non-convex, but only establish sub-linear convergence. Recently, \citet{li2016stochastic} consider SVRG with a zero-norm constraint and prove   linear convergence for this specific formulation. In comparison,  our results hold more generally,  covering not only sparsity model but also corrected Lasso with noisy covariate, group sparsity model, and beyond.

	\section{Problem setup}
	In this paper we consider two setups, namely, (1) convex but not strongly convex $F(w)$  and (2) non-convex $F(w)$.
	For the convex case, we consider 
	\begin{equation}\label{equ.formulation.convex}
	\begin{split}
	\min_{g(w)\leq \rho} F(w)&:=f(w)+\lambda 
	g(w)\\
	&:=\frac{1}{n}\sum_{i=1}^{n} f_i (w)+\lambda g(w),
	\end{split}
	\end{equation}
	where $\rho>0$ is some pre-defined radius, $f_i(w)$ is the loss function for  sample $i$, and $g(w)$ is a norm.
	Here we assume each $f_i(w)$ is convex and $L_i$-smooth.

	For the non-convex $F(w)$ we consider
	\begin{equation}\label{obj:non_convex}
	\begin{split}
	\min_{d_\lambda(w)\leq \rho}F(w):&=f(w)+d_{\lambda,\mu}(w)\\
	&=\frac{1}{n}\sum_{i=1}^{n} f_i(w)+d_{\lambda,\mu}(w) ,
	\end{split}
	\end{equation}
	where $\rho>0$ is some pre-defined radius, $f_i(w)$ is convex and $L_i$ smooth, $d_{\lambda,\mu}(w)$ is a non-convex regularizer depending on a tuning parameter $\lambda$  and a parameter $\mu$ explained in section \ref{section:Assumption_non_convex}. This M-estimator includes a side constraint depending on $d_\lambda (w)$, which needs to be a convex function and admits a lower bound $d_{\lambda}(w)\geq\|w\|_1$. It is also closed w.r.t.\ to $d_{\lambda,\mu}(w)$,  more details are deferred to section \ref{section:Assumption_non_convex}.
	
	We list some examples that belong to these two setups.
	\begin{itemize}
		\item Lasso: $F(w)=\sum_{i=1}^{n}\frac{1}{2}(y_i-w^Tx_i)^2+\lambda \|w\|_1$.
		\item Logistic regression with $\ell_1$ penalty:
		$$F(w)=\sum_{i=1}^{n} \log (1+\exp (-y_ix_i^Tw))+\lambda \|w\|_1.$$
		\item Group Lasso $F(w)=\sum_{i=1}^{n}\frac{1}{2}(y_i-w^Tx_i)^2+\lambda \|w\|_{1,2}$.
		\item Corrected Lasso \cite{loh2011high}:
		$F(w)=\sum_{i=1}^{n}\frac{1}{2n} (\langle w,x_i \rangle-y_i)^2-\frac{1}{2}w^T\Sigma w+\lambda \|w\|_1$, where $\Sigma$ is some positive definite matrix.
		\item linear regression with SCAD regularizer \cite{fan2001variable}: $F(w)=\sum_{i=1}^{n}\frac{1}{2n} (\langle w,x_i \rangle-y_i)^2+SCAD(w).$
	\end{itemize}  
	
	The first three examples belong to first setup while the last two belong to the second setup.
	
	\subsection{Restricted Strongly Convexity}
	
	Restricted Strongly Convexity (RSC) is proposed in \citet{negahban2009unified,van2009conditions} and has been explored in several other work \cite{agarwal2010fast,loh2013regularized}. We say the loss function $f(w)$ satisfies the RSC condition with curvature $\kappa$ and the tolerance parameter $\tau$ with respect to the norm $g(w)$ if 
	\begin{equation}
	\begin{split}
	& \Delta f(w_1,w_2)\triangleq f(w_1)-f(w_2)-\langle \nabla f(w_2),w_1-w_2 \rangle\\
	&\geq \frac{\kappa}{2} \|w_1-w_2\|_2^2-\tau g^2(w_1-w_2).
	\end{split}
	\end{equation}
	When $f(w)$ is $\gamma$ strongly convex,  then it is RSC with  $\kappa=\gamma$ and $\tau=0$. However in many case,  $f(w)$ may not be strongly convex, especially in the high dimensional case where the ambient dimension $p>n$. On the other hand, RSC is easier to be satisfied.  Take Lasso as an example, under some mild condition, it is shown that \cite{negahban2009unified} 
	$ \Delta f(w_1,w_2)\geq c_1 \|w_1-w_2\|_2^2-c_2\frac{\log p}{n}\|w_1-w_2\|_1^2,$
	where $c_1$, $c_2$ are some  positive constants. Besides Lasso, the RSC condition holds for a large range of statistical models including log-linear model, group sparsity model, and low-rank model. See \citet{negahban2009unified} for more detailed discussions. 
	\subsection{Assumption on convex regularizer $g(w)$}
	Decomposability is the other ingredient needed to analyze the algorithm.
	
	\mbox{\textbf{Definition:}} A regularizer is {\em decomposable} with respect to a pair of subspaces $\mathcal{A}\subseteq \mathcal{B}$ if 
	$$ g(\alpha+\beta)=g(\alpha)+g(\beta) \quad \text{for all} \quad \alpha\in \mathcal{A}, \beta\in \mathcal{B^\perp}, $$ where $\perp$ means the orthogonal complement.
	
	
	A concrete example is $l_1$ regularization for sparse vector supported on subset $S$. We define the subspace pairs with respect to the subset $S\subset\{1,...,p\}$, 
	$ \mathcal{A}(S)=\{w\in \mathbb{R}^p| w_j=0 \text{~for all~} j\notin S \} $ and $\mathcal{B}(S)=\mathcal{A}(S).$ The decomposability is thus easy to verify. Other widely used examples include non-overlap group norms such as$\|\cdot\|_{1,2}$, and the nuclear norm $\|| \cdot\||_{*}$~\cite{negahban2009unified}. In the rest of the paper, we denote $w_{\mathcal{B}}$ as the projection of $w$ on the subspace $\mathcal{B}$.
	
	\subsubsection*{Subspace compatibility constant}
	For any subspace $\mathcal{A}$ of $\mathbb{R}^p$, the {\em subspace compatibility constant }with respect to the pair $(g,\|\cdot\|_2)$ is give by 
	$$ \Psi (\mathcal{A})=\sup_{u\in \mathcal{A}\backslash\{0\}} \frac{g(u)}{\|u\|_2}.$$
	That is, it is the Lipschitz constant of the regularizer restricted in $\mathcal{A}$. For example, for the above-mentioned sparse vector with cardinality $s$, $\Psi(\mathcal{A})=\sqrt{s}$ for $g(u)=\|u\|_1$.
	
	\subsection{Assumption on non-convex regularizer $d_{\lambda,\mu} (w)$}\label{section:Assumption_non_convex}
	We consider  regularizers that are {\em separable} across coordinates, i.e., $d_{\lambda,\mu}(w)=\sum_{j=1}^{p} \bar{d}_{\lambda,\mu} (w_j)$.
	Besides the separability, we make further assumptions on the univariate function $\bar{d}_{\lambda,\mu} (t)$:
	\begin{enumerate}
		\item  $\bar{d}_{\lambda,\mu}(\cdot)$ satisfies $\bar{d}_{\lambda,\mu}(0)=0$ and is symmetric about zero (i.e.,~$\bar{d}_{\lambda,\mu}(t)=\bar{d}_{\lambda,\mu}(-t)$).
		\item On the nonnegative real line, $\bar{d}_{\lambda,\mu} (\cdot) $ is nondecreasing.
		\item For $t>0$,  $\frac{\bar{d}_{\lambda,\mu} (t)}{t}$ is nonincreasing in t.
		\item $\bar{d}_{\lambda,\mu}(\cdot)$ is differentiable for all $t\neq0$ and subdifferentiable at $t=0$, with $\lim_{t\rightarrow0^+} \bar{d}_{\lambda,\mu}'(t)=\lambda L_d$.
		\item $ \bar{d}_{\lambda}(t):= (\bar{d}_{\lambda,\mu}(t)+\frac{\mu}{2}t^2)/\lambda$ is convex.	
	\end{enumerate}
	For instance, SCAD satisfying these assumptions.
	
	$\mathbf{SCAD_{\lambda,\zeta} (t)=}$ 
	$$
	\begin{cases}
	\lambda |t|,   & \quad \text{for}~ |t|\leq \lambda,\\
	-(t^2-2\zeta\lambda|t|+\lambda^2)/(2(\zeta-1)), & \quad \text{for} ~\lambda<|t|\leq \zeta\lambda,\\
	(\zeta+1)\lambda^2/2, &\text{for} ~|t|>\zeta\lambda,
	\end{cases}
	$$
	where $\zeta>2$ is a fixed parameter. It satisfies the assumption with $L_d=1$ and $\mu=\frac{1}{\zeta-1}$ \cite{loh2013regularized}.\\
	\subsection{Applying SDCA}
	\subsubsection{Convex $F(w)$}
	Following a similar line as we did for Lasso, to apply the SDCA algorithm, we define
	$\phi_i(w)=\frac{n+1}{n} f_i(w) \quad \text{for} ~ i=1,...,n$ , $\phi_{n+1}(w)=-\frac{ \tilde{\lambda} (n+1)}{2}\|w\|_2^2 ,$ $\tilde{g}(w)= \frac{1}{2}\|w\|_2^2+\frac{\lambda}{\tilde{\lambda}}g(w).$
	Correspondingly, the new smoothness parameters are $\tilde{L}_i=\frac{n+1}{n}L_i \quad \text{for} \quad i=1,...,n  $ and  $\tilde{L}_{n+1}=\tilde{\lambda}(n+1)$. Problem~\eqref{equ.formulation.convex} is thus equivalent to the following
	\begin{equation} \label{new_objective}
	\begin{split}
	\min_{w\in\Omega}\quad \frac{1}{1+n} \sum_{i=1}^{n+1}\phi_i(w)+ \tilde{\lambda} \tilde{g}(w).
	\end{split}
	\end{equation}
	This enables us to apply Algorithm \ref{alg:proximal_SDCA} to the problem with $\phi_i$, $\tilde{g}$, $\tilde{L}_i$, $\tilde{\lambda}$.  In particular, while $\phi_{n+1}$ is not convex, $\tilde{g}$ is still convex (1-strongly convex). We exploit this property in the proof and define $\tilde{g}^*(v)=\max_{w\in \Omega} \langle w,v\rangle-\tilde{g}(w)$, where $\Omega$ is a convex compact set. Since $\tilde{g}(w)$ is $1$-strongly convex, $\tilde{g}^*(v)$ is $1$-smooth \citep[Theorem 1 of][]{nesterov2005smooth}.
	\subsubsection{Non-convex $F(w)$}
	Similarly, we define 
	\begin{align*}
	& \phi_i(w)=\frac{n+1}{n} f_i(w) \quad \text{for} \quad i=1,...,n, \\
	&\phi_{n+1}(w)=-\frac{ (\tilde{\lambda}+\mu) (n+1)}{2}\|w\|_2^2,\\
	&\tilde{g}(w)= \frac{1}{2}\|w\|_2^2+\frac{\lambda}{\tilde{\lambda}}d_\lambda(w),
	\end{align*}
	and then apply Algorithm \ref{alg:proximal_SDCA} on  
	$$ \min_{w\in \Omega}\quad \frac{1}{1+n} \sum_{i=1}^{n+1}\phi_i(w)+ \tilde{\lambda} \tilde{g}(w).$$
	The update rule of proximal step  for different $g(\cdot)$,$d_{\lambda}(\cdot)$ (such as SCAD and MCP) can be found in \cite{loh2013regularized}.

	\section{ Theoretical Result}
	In this section, we present the main theoretical results, and  some corollaries that instantiate the main results in several well known statistical models.
	\subsection{Convex $F(w)$}
	To begin with, we define several terms related to the  algorithm. 		
	\begin{itemize}
		\item  $w^*$ is true unknown parameter. $g^*(\cdot)$ is the dual norm of $g(\cdot)$. Conjugate function $\tilde{g}^*(v)=\max_{w\in \Omega} \langle w,v \rangle-\tilde{g}(w)$, where $\Omega=\{ w| g(w)\leq \rho  \}$.
		\item $\hat{w}$ is the optimal solution to problem \ref{equ.formulation.convex}, and we assume $\hat{w}$ is in the interior of $\Omega$ w.l.o.g. by choosing $\Omega$ large enough. \item $(\hat{w},\hat{v})$ is an optimal solution pair satisfying $\tilde{g}(\hat{w})+\tilde{g}^*(\hat{v})=\langle \hat{w},\hat{v} \rangle$. 
		\item $A_t=\sum_{j=1}^{n+1}\frac{1}{q_j}\|a_j^t-\hat{a}_j\|_2^2,$ where $\hat{a}_j=-\nabla \phi_j(\hat{w}).$ $B_t=2(\tilde{g}^*(v^t)-\langle \nabla \tilde{g}^*(\hat{v}),v^t-\hat{v} \rangle-\tilde{g}^*(\hat{v})).$
	\end{itemize}
	We remark that our definition on the first potential $A_t$ is the same as in \citet{shalev2016sdca}, while the second one $B_t$ is different. If $ \tilde {g} (w)=\frac{1}{2}\|w\|_2^2$, our definition on $B_t$ reduces to that in \cite{shalev2016sdca}, i.e.,  $ \|w^t-\hat{w}\|_2^2 $.  To see this, when $ \tilde {g} (w)=\frac{1}{2}\|w\|_2^2$, $\tilde{g}^*(v)=\frac{1}{2} \|v\|_2^2$ and $v^t=w^t, \hat{v}=\hat{w}$. We then define another potential $C_t$, which is a combination of $A_t$ and $B_t$. $ C_t=\frac{\eta}{(n+1)^2} A_t +\frac{\tilde{\lambda}}{2} B_t.$ Notice that using smoothness of $\tilde{g}^*(\cdot)$ and Lemma \ref{Lemma.smooth} and \ref{Lemma.conjugate} in the appendix, it is not hard to show $ B_t\geq \| w^t-\hat{w} \|_2^2. $ Thus if $C_t$ converges, so does $\|w^t-\hat{w}\|_2^2 $.
	
	For notational simplicity, we  define the following two terms used in the theorem.
	\begin{itemize}
		\item Effective RSC parameter: $\tilde{\kappa}=\kappa-64\tau \Psi^2 ({\mathcal{B}}). $ 
		\item  Tolerance: $\delta=c_1\tau \delta_{stat}^2,$ where $ \delta_{stat}^2=(\Psi({\mathcal{B}})\|\hat{w}-w^* \|_2+g (w^*_{\mathcal{A}^{\perp}}))^2$, $c_1$ is a universal positive constant.
	\end{itemize}

	\begin{theorem}\label{main_theorem}
		Assume each $f_i(w)$ is $L_i$ smooth and convex, $f(w)$ satisfies the RSC condition with parameter $(\kappa,\tau)$, $ w^*$ is feasible, the regularizer is decomposable w.r.t $( \mathcal{A}, \mathcal{B} )$ such that $\tilde{\kappa}\geq 0$,  and the Algorithm \ref{alg:proximal_SDCA} runs with $\eta \leq \min\{\frac{1}{16 (\tilde{\lambda}+\bar{L})},\frac{1}{4\tilde{\lambda} (n+1)}\}$, where $\bar{L}=\frac{1}{n}\sum_{i=1}^{n} L_i$, $\tilde{\lambda} $ is chosen such that  $0<\tilde{\lambda}\leq \tilde{\kappa}$.  If we choose the regularization parameter $\lambda$ such that $\lambda\geq \max(2g^*(\nabla f(w^*)),  c \tau \rho) $ where $c$ is some universal positive constant, then we have 
		$$ \mathbb{E}(C_t)\leq (1-\eta \tilde{\lambda})^t C_0,$$
		until $ F(w^t)-F(\hat{w})\leq \delta $, where the expectation is for the randomness of sampling of $i$ in the algorithm. 
	\end{theorem}
	Some remarks are in order for interpreting the theorem.
	
	\begin{enumerate}
		\item In several statistical models, the requirement of $\tilde{\kappa}> 0$ is easy to satisfy under mild conditions. For instance, in Lasso  we have $\Psi^2 ({\mathcal{B}})=s$. $\tau\simeq \frac{\log p}{n}$, $\kappa=\frac{1}{2}$ if the feature vector $x_i$ is sampled from $N(0,I)$. Thus, if ~$ \frac{128 s\log p}{n}\leq 1 $, we have $\tilde{\kappa}\geq \frac{1}{4}$.  
		\item In some models, we can choose the pair of subspace $ (\mathcal{A}, \mathcal{B})$ such that $ w^*\in \mathcal{A}$ and thus $ \delta=c_1 \tau \Psi^2(\mathcal{B}) \| \hat{w}-w^*\|_2^2 $. In Lasso, as we mentioned above $\tau\simeq \frac{\log p}{n}$, thus $\delta\simeq c_1s \frac{\log p}{n} \|\hat{w}-w^*\|_2^2 $, i.e., this tolerance is dominated by statistical error if $s$ is small and $n$ is large. 
		\item  We know $B_t\geq \|w^t-\hat{w}\|_2^2$ and $A_t\geq 0$, thus $C_t\geq \frac{\tilde{\lambda}}{2}\|w^t-\hat{w}\|_2^2$ , thus, if $C_t$ converge, so does $\|w^t-\hat{w}\|_2^2.$ 
		When $F(w^t)-F(\hat{w})\leq \delta$, using Lemma \ref{lemma.RSC_cone} in the supplementary material, it is easy to get $ \|w^t-\hat{w}\|_2^2\leq \frac{4\delta}{\tilde{\kappa}}.$  
	\end{enumerate}
	Combining these remarks, Theorem~\ref{main_theorem} states that the optimization error decreases geometrically until it achieves the tolerance $\delta/\tilde{\kappa}$ which is  dominated by the statistical error $\|\hat{w}-w^*\|_2^2$, thus can be ignored from the statistical view. 
	
	
	If $g(w)$ in Problem \eqref{ERM} is indeed 1-strongly convex, we have the following proposition, which extends dual-free SDCA \cite{shalev2016sdca} into the general regularization case. Notice we now directly apply the Algorithm \ref{alg:proximal_SDCA} on Problem \eqref{ERM} and change the definitions of $A_t$ and $B_t$ correspondingly. In particular, $ A_t=\sum_{j=1}^{n}\frac{1}{q_j}\|a_j^t-\hat{a}_j\|_2^2$, $B_t=2(g^*(v^t)-\langle \nabla g^*(\hat{v}),v^t-\hat{v} \rangle-g^*(\hat{v}))$, where $\hat{a}_j=-\nabla f_j(\hat{w})$. $C_t$ is still defined in the same way, i.e., $ C_t=\frac{\eta}{n^2} A_t +\frac{\lambda}{2} B_t.$

	\begin{proposition}
		Suppose each $f_i (w)$ is $L_i$ smooth, $f(w)$ is convex, $g(w)$ is 1- strongly convex,  $\hat{w}$ is the optimal solution of  Problem \eqref{ERM}, and $\bar{L}=\frac{1}{n} \sum_{i=1}^{n}L_i $. Then we have:
		(I) If each $f_i(x)$ is convex, we run the Algorithm 
		\ref{alg:proximal_SDCA}  with $\eta \leq \min\{ \frac{1}{4\bar{L}},\frac{1}{4\lambda n} \}$, then
		$ \mathbb{E}(C_t)\leq (1-\eta \lambda )^t C_0. $  (II) Otherwise, we run the Algorithm \ref{alg:proximal_SDCA} with  $\eta \leq \min\{ \frac{\lambda}{4\bar{L}^2},\frac{1}{4\lambda n} \}$, and $ \mathbb{E}(C_t)\leq (1-\eta \lambda )^t C_0. $ Note that $ B_t\geq \|w^t-\hat{w}\|_2^2 $, $A_t\geq 0$, thus $\|w^t-\hat{w}\|_2^2$ decreases geometrically.
	\end{proposition}
	In the following we present several corollaries that instantiate Theorem~\ref{main_theorem} with several concrete statistical models. This essentially requires to choose appropriate subspace pair $(\mathcal{A},\mathcal{B})$ in these models and check the RSC condition.
	
	\subsubsection{Sparse linear regression}
	Our first example is Lasso, where $f_i(w)=\frac{1}{2}(y_i-w^Tx_i)^2$ and  $g(w)=\|w\|_1 $. We assume each feature vector $x_i$ is generated from  Normal distribution $N(0,\Sigma)$ and the true parameter $w^* \in \mathbb{R} ^p$ is sparse with cardinality $s$. The observation $y_i$ is generated by $y_i=(w^*)^T x_i+\xi_i$, where $\xi_i$ is a Gaussian noise with mean $0$ and variance  $\sigma^2$. We denote the data matrix by $X\in \mathbb{R}^{n\times p}$ and $X_j$ is the $j$th column of $X$. Without loss of generality,  we assume $X$ is column normalized, i.e., $ \frac{\|X_j\|_2}{\sqrt{n}}\leq 1$  for all $j=1,2,...,p.$  We denote $\sigma_{\min} (\Sigma)$ as the smallest eigenvalue of $\Sigma$, and $\nu (\Sigma)=\max_{i=1,2,..,p} \Sigma_{ii} $.  
	\begin{corollary}\label{cor.lasso}
		Assume $w^*$ is the true parameter supported on a subset with cardinality at most $s$, and we choose the parameter $\lambda,\tilde{\lambda} $ such that $\lambda \geq \max(6\sigma\sqrt{\frac{\log p }{n}},  c_1  \rho \nu (\Sigma) \frac{\log p}{n})$ and $0< \tilde{\lambda} \leq \tilde{\kappa}$  hold, where  $\tilde{\kappa}=\frac{1}{2} \sigma_{\min}(\Sigma)-c_2\nu (\Sigma) \frac{s \log p}{n}$, where $c_1,c_2$ are some universal positive constants. Then we run the Algorithm \ref{alg:proximal_SDCA} with $\eta \leq \min\{ \frac{1}{16(\tilde{\lambda} +\bar{L})},\frac{1}{4\tilde{\lambda} (n+1)} \}$ and have 
		$$ \mathbb{E}(C_t)\leq (1-\eta \tilde{\lambda})^t C_0 $$
		with probability at least $1-\exp(-3\log p)-\exp (-c_3 n),$
		until $F(w^t)-F(\hat{w})\leq \delta,$ where   $ \delta=c_4 \nu(\Sigma) \frac{s\log p }{n} \|\hat{w}-w^*\|_2^2$. $c_1,c_2,c_3,c_4$ are some universal positive constants.
	\end{corollary}
	The  requirement $\lambda\geq 6\sigma\sqrt{\frac{\log p }{n}}$ is documented in literature  \cite{negahban2009unified} to ensure that Lasso is statistically consistent. And $\lambda \geq c_1 \ \rho \nu (\Sigma) \frac{\log p}{n}$ is needed for fast convergence of optimization algorithms, which is similar to the condition proposed in   in \citet{agarwal2010fast} for batch optimization algorithm. When $\frac{s\log p}{n}=o(1)$, which is necessary for statistical consistency of Lasso, we have $\frac{1}{2}\sigma_{\min} (\Sigma)-c_2 \nu (\Sigma) \frac{s \log p}{n}\geq 0$ , which guarantees the existence of  $\tilde{\lambda}$.  Also notice under this condition,  $\delta= c_3\nu(\Sigma)\frac{s\log p}{n} \|\hat{w}-w^*\|_2^2$  is of a lower order of $ \|\hat{w}-w^*\|_2^2$. Using remark 3 in Theorem \ref{main_theorem}, we have  $ \|w^t-\hat{w}\|_2^2\leq \frac{4\delta}{\tilde{\kappa}} $, which is dominated by the statistical error $ \|\hat{w}-w^*\|_2^2 $ and hence  can be ignored from the statistical perspective. Thus to sum up, Corollary \ref{cor.lasso} states the optimization error decreases geometrically until it achieves the statistical limit of Lasso.

	\subsubsection{Group Sparsity Model}
	\citet{yuan2006model} introduce the group Lasso to allow predefined groups of covariates  to be selected together into or out of a model together. The most commonly used regularizer to encourage group sparsity is $\|\cdot\|_{1,2} $. In the following, we define group sparsity formally. We assume groups are disjointed, i.e., $\mathcal{G}=\{G_1,G_2,...,G_{N_{\mathcal{G}}}\}$ and $G_i \cap G_j=\emptyset$. The regularization is $\|w\|_{\mathcal{G},q}\triangleq \sum_{g=1}^{N_{g}} \|w_g\|_q $. When $q=2$, it reduces to the commonly used group Lasso \cite{yuan2006model}, and another popularly used case is $q=\infty$ \cite{turlach2005simultaneous,quattoni2009efficient}. 	We require the following condition, which  generalizes the column normalization condition in the Lasso case. Given a group $G$ of size $m$ and $X_G\in \mathbb{R}^{n\times m}$ , the associated operator norm $||| X_{G_i}|||_{q \rightarrow 2} \triangleq \max_{\|w\|_q=1}\|X_Gw\|_2$ satisfies
	$$ \frac{||| X_{G_i}|||_{q\rightarrow 2}}{\sqrt{n}}\leq 1 ~~\text{for all} ~~ i=1,2,...,N_{\mathcal{G}}.$$ 
	The condition reduces to the column normalized condition when each group contains only one feature (i.e., Lasso).
	
	We now define the subspace pair $(\mathcal{A},\mathcal{B})$ in the group sparsity model. For a subset $S_{\mathcal{G}}\subseteq \{1,...,N_{\mathcal{G}} \}$ with cardinality $s_{\mathcal{G}}= |S_{\mathcal{G}}|$, we define the subspace 
	$$ \mathcal{A} (S_{\mathcal{G}})=\{w|w_{G_{i}}=0 ~~ \text{for all } ~~ i\notin S_{\mathcal{G}} \} , $$and $\mathcal{A}=\mathcal{B}$.
	The orthogonal complement is 
	$$ \mathcal{B}^{\perp}(S_{\mathcal{G}}) = \{w|w_{G_{i}}=0 ~~ \text{for all } ~~ i\in S_{\mathcal{G}} \}. $$ We can easily verify that 
	$$ \|\alpha +\beta\|_{\mathcal{G},q}=\|\alpha\|_{\mathcal{G},q}+\|\beta\|_{\mathcal{G},q}, $$
	for any $\alpha\in \mathcal{A}(S_{\mathcal{G}}) $ and $\beta\in \mathcal{B}^{\perp} (S_{\mathcal{G}})$.
	
	In the following corollary, we use $q=2$, i.e., group Lasso, as an example. We assume the observation $y_i$ is generated by $y_i=x_i^T w^*+\xi_i $, where $x_i\sim N(0,\Sigma)$, and $\xi_i\sim N(0,\sigma^2)$.
	\begin{corollary}[Group Lasso]\label{cor.group_lasso}
		Assume $w\in \mathbb{R}^{p}$ and each group has $m$ parameters, i.e.,  $p=m N_{\mathcal{G}}$. Denote by $s_{\mathcal{G}}$  the cardinality of non-zero group,  and we choose   parameters $\lambda, \tilde{\lambda}$ such that  
		\begin{equation*}\begin{split}
		\lambda&\geq \max \big(4\sigma (\sqrt{\frac{m}{n}}+\sqrt{\frac{\log N_{\mathcal{G}}}{n}}), c_1\rho \sigma_2(\Sigma) (\sqrt{\frac{m}{n}}+\sqrt{\frac{3 \log N_{\mathcal{G}}}{n}} )^2\big);\\
		\mbox{and}\quad 0 &< \tilde{\lambda}\leq \tilde{\kappa}, \quad \mbox{where}\,\, \tilde{\kappa}=\sigma_1(\Sigma)-c_2\sigma_2(\Sigma)s_{\mathcal{G}} (\sqrt{\frac{m}{n}}+\sqrt{\frac{3 \log N_{\mathcal{G}}}{n}} )^2;
		\end{split}\end{equation*}
		where  $\sigma_1 (\Sigma)$ and $\sigma_2 (\Sigma)$ are   positive constant  depending only on $\Sigma$. If we run the Algorithm 1 with $\eta \leq \min\{ \frac{1}{16(\tilde{\lambda} +\bar{L})},\frac{1}{4\tilde{\lambda} (n+1)} \}$, then we have
		
		$$ \mathbb{E}(C_t)\leq (1-\eta \tilde{\lambda})^t C_0 $$  with probability at least $1-2\exp (-2 \log N_{\mathcal{G}})-c_2 \exp(-c_3n) $ ,
		until $F(w^t)-F(\hat{w})\leq \delta,$ where   $ \delta=c_3\sigma_2 (\Sigma)s_{\mathcal{G}}\big(\sqrt{\frac{m}{n}}+\sqrt{\frac{3 \log N_{\mathcal{G}}}{n}} \big)^2 \|\hat{w}-w^*\|_2^2$. 
		
	\end{corollary}
	We offer some discussions to interpret the corollary.
	To satisfy the requirement $\tilde{\kappa}\geq 0$, it suffices to have
	$$ s_{\mathcal{G}} \left(\sqrt{\frac{m}{n}}+\sqrt{\frac{3 \log N_{\mathcal{G}}}{n}} \right)^2=o(1).$$ This is a mild condition, as it is needed to guarantee the statistical consistency of group Lasso \cite{negahban2009unified}. Notice that the condition is easily   satisfied  when $s_{\mathcal{G}}$ and $m$ are small. Under this same condition, since 
	$ \delta=c_3\sigma_2(\Sigma)s_{\mathcal{G}}\big(\sqrt{\frac{m}{n}}+\sqrt{\frac{3 \log N_{\mathcal{G}}}{n}} \big)^2 \|\hat{w}-w^*\|_2^2, $
	we conclude that $ \delta$ is dominated by $ \|\hat{w}-w^*\|_2^2.$
	Again, it implies the optimization error decrease geometrically up to the scale $o(\|\hat{w}-w^*\|_2^2)$ which is dominated by the statistical error of the model.
	
		\subsubsection{Extension to generalized linear model }
		We consider the {\em generalized linear model}  of the following form,
		$$ \min_{w\in \Omega} \frac{1}{n} \sum_{i=1}^{n} (\Phi (w ,x_i)-y_i\langle w,x_i \rangle)+\lambda \|w\|_1,$$ which covers such  case  as Lasso (where $\Phi (\theta)=\frac{\theta^2}{2}$) and logistic regression (where $\Phi(\theta)=\log (1+\exp (\theta))$). In this model, we have $$\Delta f(w_1,w_2)=\frac{1}{n} \sum_{i=1}^{n} \Phi'' (\langle w_t,x_i \rangle) \langle x_i, w_1-w_2 \rangle^2,$$ where $ w_t=t w_1+(1-t)w_2$ for some $t\in [0,1 ].$
		The RSC condition   thus is equivalent to:
		\begin{equation}
		\begin{split}
		&\frac{1}{n} \sum_{i=1}^{n} \Phi'' (\langle w_t,x_i \rangle)\langle x_i, w_1-w_2 \rangle ^2 \\
		\geq &\frac{\kappa}{2} \|w_1-w_2\|_2^2-\tau g^2 (w_1-w_2) ~~\text{for}~~ w_1,w_2\in \Omega.
		\end{split}
		\end{equation}
		Here we require $\Omega$ to be a bounded set \cite{loh2013regularized}. This requirement is essential since in some generalized linear model $\Phi''(\theta)$ approaches to zero as $ \theta$ diverges. For instance, in logistic regression, $\Phi''(\theta)=\exp (\theta)/(1+\exp(\theta))^2$, which tends to zero as $\theta\rightarrow \infty$ . For a broad class of generalized linear models, RSC holds with $\tau=c\frac{\log p}{n}$, thus the same result as that of Lasso holds, modulus change of constants.

	\subsection{Non-convex $F(w)$}
	In the non-convex case, we assume the following RSC condition: 
	$$ \Delta f(w_1,w_2)\geq \frac{\kappa}{2}\|w_1-w_2\|_2^2-\tau \|w_1-w_2\|_1^2 $$
	with $\tau=\theta \frac{\log p}{n}$ for some constant $\theta$. We again define the potential $A_t$ ,$B_t$ and $C_t$ in the same way with convex case. The main difference  is that now we have $\phi_{n+1}=-\frac{(\tilde{\lambda}+\mu)(n+1)}{2}\|w\|_2^2$  and the effective RSC parameter $\tilde{\kappa}$ is different. The necessary notations for presenting the theorem are listed below:
	\begin{itemize}
		\item $w^*$ is  the unknown true parameter that is $s$-sparse. Conjugate function $\tilde{g}^*(v)=\max_{w\in \Omega} \langle w,v \rangle-\tilde{g}(w)$, where $\Omega=\{ w| d_{\lambda}(w)\leq \rho  \}.$ Note $\Omega$ is convex due to convexity of $d_\lambda(w)$.
		\item $\hat{w}$ is the global optimum of Problem \eqref{obj:non_convex}, we assume it is in the interior of $\Omega$ w.l.o.g.   
		\item $(\hat{w},\hat{v})$ is an optimal solution pair satisfying $\tilde{g}(\hat{w})+\tilde{g}^*(\hat{v})=\langle \hat{w},\hat{v} \rangle$. 
		\item $A_t=\sum_{j=1}^{n+1}\frac{1}{q_j}\|a_j^t-\hat{a}_j\|_2^2,$ $B_t=2(\tilde{g}^*(v^t)-\langle \nabla \tilde{g}^*(\hat{v}),v^t-\hat{v} \rangle-\tilde{g}^*(\hat{v})),$ $ C_t=\frac{\eta}{(n+1)^2} A_t +\frac{\tilde{\lambda}}{2} B_t,$ where $\hat{a}_j=-\nabla \phi_j(\hat{w}).$
		\item Effective RSC parameter: $\tilde{\kappa}=\kappa-\mu-64\tau s$, where $\tau=\theta\frac{\log p}{n}$ for some constant $\theta$. Tolerance: $\delta=c_1\tau s \|\hat{w}-w^*\|_2^2$, where $c_1$ is a universal positive constant.
	\end{itemize}
	
	\begin{theorem}
		Suppose $w^*$ is $s$ sparse, $\hat{w}$ is the global optimum of Problem \eqref{obj:non_convex}. Assume each $f_i(w)$ is $L_i$ smooth and convex, f(w) satisfies the RSC condition with parameter $(\kappa,\tau)$, where $\tau=\theta\frac{\log p}{n}$ for some constant $\theta$, $d_{\lambda,\mu} (w)$ satisfies the Assumption in section \ref{section:Assumption_non_convex},   $\lambda L_d\geq \max\{ c\rho\theta \frac{\log p }{n} , 4 \|\nabla f (w^*)\|_\infty\} $ where $c$ is some universal positive constant, and $\tilde{\kappa}-\mu\geq \tilde{\lambda}> 0, $ the Algorithm \ref{alg:proximal_SDCA} runs with $\eta \leq \min\{\frac{1}{16 (\tilde{\lambda}+\bar{L})},\frac{1}{4\tilde{\lambda} (n+1)}\}$, where $\bar{L}=\frac{1}{n}\sum_{i=1}^{n} L_i$, then we have
		
		$$ \mathbb{E}(C_t)\leq (1-\eta \tilde{\lambda})^t C_0,$$
		until $ F(w^t)-F(\hat{w})\leq \delta $, where the expectation is for the randomness of sampling of $i$ in the algorithm.
		
	\end{theorem}
	Some remarks are in order to interpret the theorem.
	\begin{itemize}
		\item We require $\tilde{\lambda}$ to satisfy $ 0<\tilde{\lambda}\leq \kappa-2\mu-64\theta\frac{\log p}{n} s$.  Thus if the non-convex parameter $\mu$ is too large, we can not find such $\tilde{\lambda}$.
		\item Note that $\delta=c_1s\theta\frac{\log p}{n}\|\hat{w}-w^*\|_2^2$ is dominated by $\|\hat{w}-w^*\|_2^2$ when the model is sparse and $n$ is large. Similar as the convex case, by $B_t\geq \|w^t-\hat{w}\|_2^2$, the theorem says the optimization error deceases linearly up to the fundamental statistical error of the model.
	\end{itemize}
	The first non-convex model we consider is  linear regression with SCAD.  Here $f_i(w)=\frac{1}{2} (\langle w, x_i \rangle-y_i)^2$ and $d_{\lambda,\mu}(\cdot)$ is $SCAD(\cdot)$ with parameter $\lambda$ and $\zeta$. The  data $(x_i,y_i)$ are generated similarly as in the Lasso case.  
	\begin{corollary}[Linear regression with SCAD regularization]	\label{cor.scad}
		Suppose we have n i.i.d. observations $ \{ (x_i,y_i) \}$, $w^*$ is $s$ sparse , $\hat{w}$ is global optima and we choose $\lambda$ and $\tilde{\lambda}$ such that $\lambda \geq \max\{ c_1\rho \nu (\Sigma) \frac{\log p }{n} , 12\sigma \sqrt{\frac{\log p}{n}}\}$, $\tilde{\kappa} -\frac{1}{\zeta-1}\geq \tilde{\lambda}>0$, then we run the algorithm with $\eta \leq \min\{\frac{1}{16 (\tilde{\lambda}+\bar{L})},\frac{1}{4\tilde{\lambda} (n+1)}\}$, where $\bar{L}=\frac{1}{n}\sum_{i=1}^{n} L_i$, then  have	
		$$ \mathbb{E}(C_t)\leq (1-\eta \tilde{\lambda})^t C_0,$$
		with $1-\exp(-3\log p)-\exp (-c_2 n),$  until $ F(w^t)-F(\hat{w})\leq \delta $, where $\tilde{\kappa}=\frac{1}{2}\sigma_{\min}(\Sigma)-c_3 \nu(\Sigma) \frac{s \log p}{n}-\frac{1}{\zeta-1}$, $\delta=c_4\nu (\Sigma)\frac{s\log p}{n} \|\hat{w}-w^*\|_2^2.$  Here $c_1, c_2, c_3,c_4$ are universal positive constants.
	\end{corollary}
	
	If the non-convex parameter $\frac{1}{\zeta-1}$ is small, $s$ is sparse, $n$ is large, then we can choose a positive $\tilde{\lambda}$ to guarantee the convergence of algorithm. Under this setting, we have the tolerance $\delta$ dominated by $\|\hat{w}-w^*\|_2^2$.
	
	The second example is the {\em corrected Lasso}. In many applications, the covariate may be observed subject to corruptions. 	In this section, we consider corrected Lasso proposed by \cite{loh2011high}. Suppose data are generated according to a   linear model $y_i=x_i^Tw^*+\xi_i,$ where $\xi_i$ is a random zero-mean sub-Gaussian noise with variance $\sigma^2 $; and   each data point $x_i$ is i.i.d.\  sampled from a zero-mean normal distribution, i.e., $x_i\sim N(0,\Sigma)$. We denote the data matrix by $X\in \mathbb{R}^{n\times p}$ , the smallest eigenvalue of $\Sigma$ by $\sigma_{\min} (\Sigma)$, and the largest eigenvalue of $\Sigma$ by $\sigma_{\max} (\Sigma) $ . 
	
	The observation $z_i$ of $x_i$ is corrupted by addictive noise, in particular,   $z_i=x_i+\varsigma_i$, where $\varsigma_i\in \mathbb{R}^p$ is a random vector independent of $x_i$, say zero-mean with known covariance matrix $\Sigma_\varsigma$. Define $\hat{\Gamma}=\frac{Z^TZ}{n}-\Sigma_\varsigma$ and $\hat{\gamma}=\frac{Z^Ty}{n}$. Our goal is to estimate $w^*$ based on $y_i$ and $z_i$ (but not $x_i$ which is not observable), and the corrected Lasso proposes to solve the following: 
	$$\hat{w}\in \arg \min_{\|w\|_1\leq \rho} \frac{1}{2} w^T \hat{\Gamma} w-\hat{\gamma} w+ \lambda \|w\|_1. $$
	Equivalently, it solves 
	$$ \min_{\|w\|_1\leq \rho}\frac{1}{2n}\sum_{i=1}^{n} (y_i-w^Tz_i)^2-\frac{1}{2}w^T\Sigma_\varsigma w +\lambda \|w\|_1 .$$
	Notice that due to the term $-\frac{1}{2}w^T\Sigma_\varsigma w $, the optimization problem is non-convex.
	\begin{corollary}[Corrected Lasso]\label{cor.corrected_lasso}
		Suppose we are given i.i.d. observation $\{(z_i,y_i)   \}$ from the linear model with additive noise,  $w^* $ is s sparse, and $\Sigma_\varsigma=\gamma_\varsigma I$. Let $\hat{w}$ be the global optima.  We choose $\lambda \geq \max\{ c_0\rho\frac{\log p }{n} , c_1 \varphi \sqrt{\frac{\log p}{n}}\} $ , and positive $\tilde{\lambda}$ such that $\tilde{\kappa}-\gamma_\varsigma\geq \tilde{\lambda}>0$, where $\varphi=(\sqrt{\sigma_{\max} (\Sigma)}+\sqrt{\gamma_\varsigma})(\sigma+\sqrt{\gamma_\varsigma}\|w^*\|_2)$, and  $$\tilde{\kappa}=\frac{1}{2} \sigma_{\min}(\Sigma)-c_2 \sigma_{\min}(\Sigma)\max \left( (\frac{\sigma_{\max} (\Sigma)+\gamma_\varsigma}{\sigma_{\min}(\Sigma)})^2,1 \right)  \frac{s\log p}{n}-\gamma_\varsigma,$$then if we run the algorithm with $\eta \leq \min\{ \frac{1}{16(\tilde{\lambda} +\bar{L})},\frac{1}{4\tilde{\lambda} (n+1)} \}$,  we have 
		$$ \mathbb{E}(C_t)\leq (1-\eta \tilde{\lambda})^t C_0,$$
		with probability at least  $1-c_3 \exp \left(-c_4 n\min \big( \frac{\sigma^2_{\min} (\Sigma)}{( \sigma_{\max}(\Sigma)+\gamma_\varsigma)^2},1 \big)    \right)-c_5\exp(-c_6 \log p)$ , until $F(w^t)-F(\hat{w})\leq \delta$, where $\delta=c_7 \sigma_{\min}(\Sigma)\max \left( (\frac{\sigma_{\max} (\Sigma)+\gamma_\varsigma}{\sigma_{\min}(\Sigma)})^2,1 \right)\frac{s\log p}{n} \|\hat{w}-w^*\|_2^2, $ and $c_0$ to $c_7$ are some universal positive constants.
	\end{corollary}
	Some remarks of the corollary are in order.
	\begin{itemize}
		\item The result can be easily extended to more general $\Sigma_\varsigma\preceq \gamma_\varsigma I.$
		\item The requirement of $\lambda$ is similar with its batch counterpart \cite{loh2011high}.
		\item Similar to the setting in Lasso, we need $\frac{s\log p}{n}=o(1)$.  To ensure the existence of such $\tilde{\lambda}$, the non-convex parameter $\gamma_{\varsigma}$ can not be too large, which is similar to the result in \cite{loh2013regularized}.
	\end{itemize}

		\section{Experiment Result}
		We report numerical experiment results to validate our theoretical findings, namely, without strong convexity, SDCA still   achieves linear convergence   under our setup.   The setup of  experiment is similar to that in \cite{qu2016linear}.  On both synthetic and real datasets we report results of SDCA and compare with several other algorithms.  These algorithms  are Prox-SVRG \cite{xiao2014proximal}, SAGA~\cite{defazio2014saga}, Prox-SAG an proximal version of the algorithm in \cite{schmidt2013minimizing}, proximal stochastic gradient (Prox-SGD), regularized dual averaging method (RDA) \cite{xiao2010dual} and the proximal full gradient method (Prox-GD) \cite{nesterov2013introductory}. For the algorithms with a constant learning rate (i.e., Prox-SAG, SAGA, Prox-SVRG, SDCA, Prox-GD), we tune the learning rate from an exponential grid  $\{ 2, \frac{2}{2^1},...,\frac{2}{2^{12}} \}$ and chose the one with best performance. Below are some further remarks. 
		
		\begin{itemize}
			
	    \item Convergence rates of Prox-SGD and RDA are sub-linear according to the analysis in \citet{bottou2010large,xiao2010dual}. The stepsize of Prox-SGD is set as $\eta_k=\eta_0/\sqrt{k}$ as suggested in \citet{duchi2009efficient} and that of RDA is $\beta_k=\beta_0 \sqrt{k}$ as suggested in \citet{xiao2010dual}. $\eta_0$  and $\beta_0$ are chosen as the one that attains the best performance among powers of 10.
		\item \citet{agarwal2010fast} prove that  Prox-GD converges linearly in the setting of our experiment.
	   \item  The convergence rate of Prox-SVRG and SAGA is linear in our setting, shown recently in \citet{qu2016linear,qu2017saga}.
	   \item To the best of our knowledge, linear convergence of  Prox-SAG in our setting has not been established, although in practice it works well.  
			\end{itemize}
	
	\subsection{Synthetic dataset}
	In this section we test algorithms above on Lasso, Group Lasso, Corrected Lasso and  SCAD .
	
	\subsubsection{Lasso}
	We generated the feature vector $x_i\in \mathbb{R}^{p}$ independently from $N(0,\Sigma)$, where we set $\Sigma_{ii}=1, ~~\text{for}~~ i=1,...,p$ and $\Sigma_{ij}=b, ~~\text{for}~~ i\neq j$.  The responds $y_i$ is generated by follows: $ y_i=x_i^T w^*+\xi_i.$  $w^*\in \mathbb{R}^p$ is a sparse vector with cardinality $s$,  where the non-zero entries are $\pm 1 $ drawn from the Bernoulli distribution with probability $0.5$. The noise $\xi_i$ follows the standard normal distribution. We set $\lambda=0.05$ in Lasso and choose $\tilde{\lambda}=0.25$ in SDCA. In the following experiment, we set $p=5000$, $n=2500$ and try different settings on $s$ and $b$.
	
	\begin{figure}
		\begin{subfigure}[b]{0.5\textwidth}
			\includegraphics[width=\textwidth]{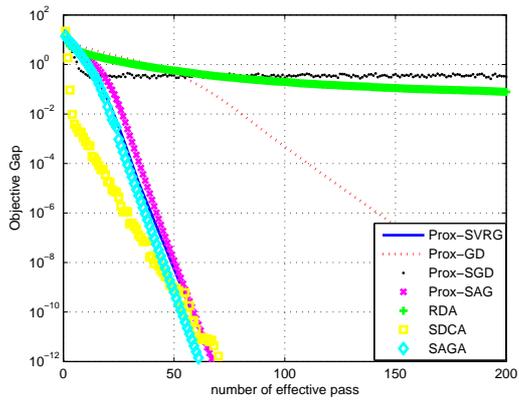}
			\caption{s=50, b=0}
		\end{subfigure}
		\begin{subfigure}[b]{0.5\textwidth}
			\includegraphics[width=\textwidth]{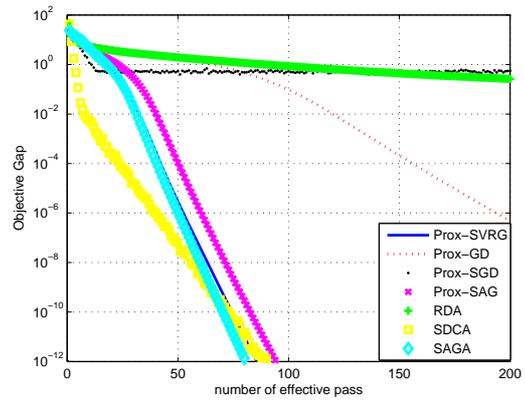}
			\caption{s=100,b=0}
		\end{subfigure}
		\begin{subfigure}[b]{0.5\textwidth}
			\includegraphics[width=\textwidth]{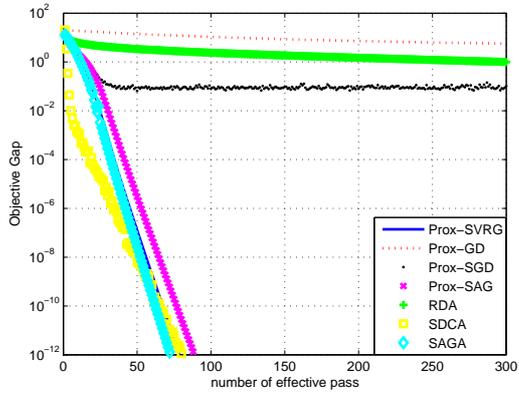}
			\caption{s=50,b=0.1}
		\end{subfigure}\quad
		\begin{subfigure}[b]{0.5\textwidth}
			\includegraphics[width=\textwidth]{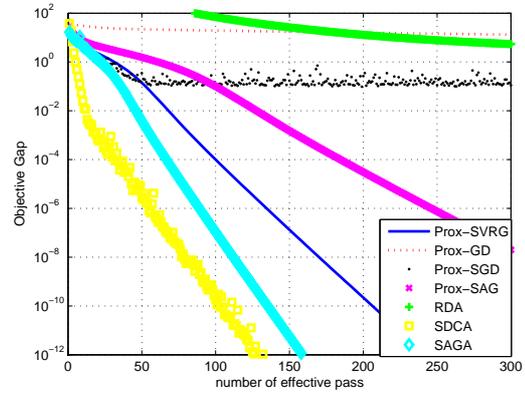}
			\caption{s=100,b=0.4}
		\end{subfigure}
		\caption{Comparison between six algorithms on Lasso. The x-axis is the number of pass over the dataset. y-axis is the objective gap $F(w_k)-F(\hat{w})$ with a log scale. In figure (a), s=50,b=0 . In figure (b), s=100,b=0. In figure (c), s=50,b=0.1. In figure (d), s=100,b=0.4 . } \label{Fig:lasso}
	\end{figure}
	
	Figure \ref{Fig:lasso} presents simulation results on Lasso. Among all four settings, SDCA, SVRG and SAG converge with linear rates.  When $b=0.4$ and $s=100$, SDCA outperforms the other two.  When $b=0$, the Prox-GD works although with slower rate. While $b$ is nonzero, the Prox-GD does not works well due to the large condition number. RDA and SGD converges slowly in all setting because of the large variance in gradient. 
	
	\subsubsection{Group Lasso}
	
	We report the experiment result on Group Lasso in Figure \ref{gp_lasso}. Similarly as Lasso, we generate the observation $ y_i=x_i^T w^*+\xi_i$ with the feature vectors independently sampled from $N(0,\Sigma)$, where $\Sigma_{ii}=1$ and $\Sigma_{ij}=b, i\neq j$. The cardinality of non-zero group is $s_{\mathcal{G}}$, and the non-zero entries are sampled uniformly from $[-1, 1]$.  In the following experiment, we try different setting on $b$, group size $m$ and group sparsity $s_{\mathcal{G}}$.

	\begin{figure}
		\begin{subfigure}[b]{0.5\textwidth}
			\includegraphics[width=\textwidth]{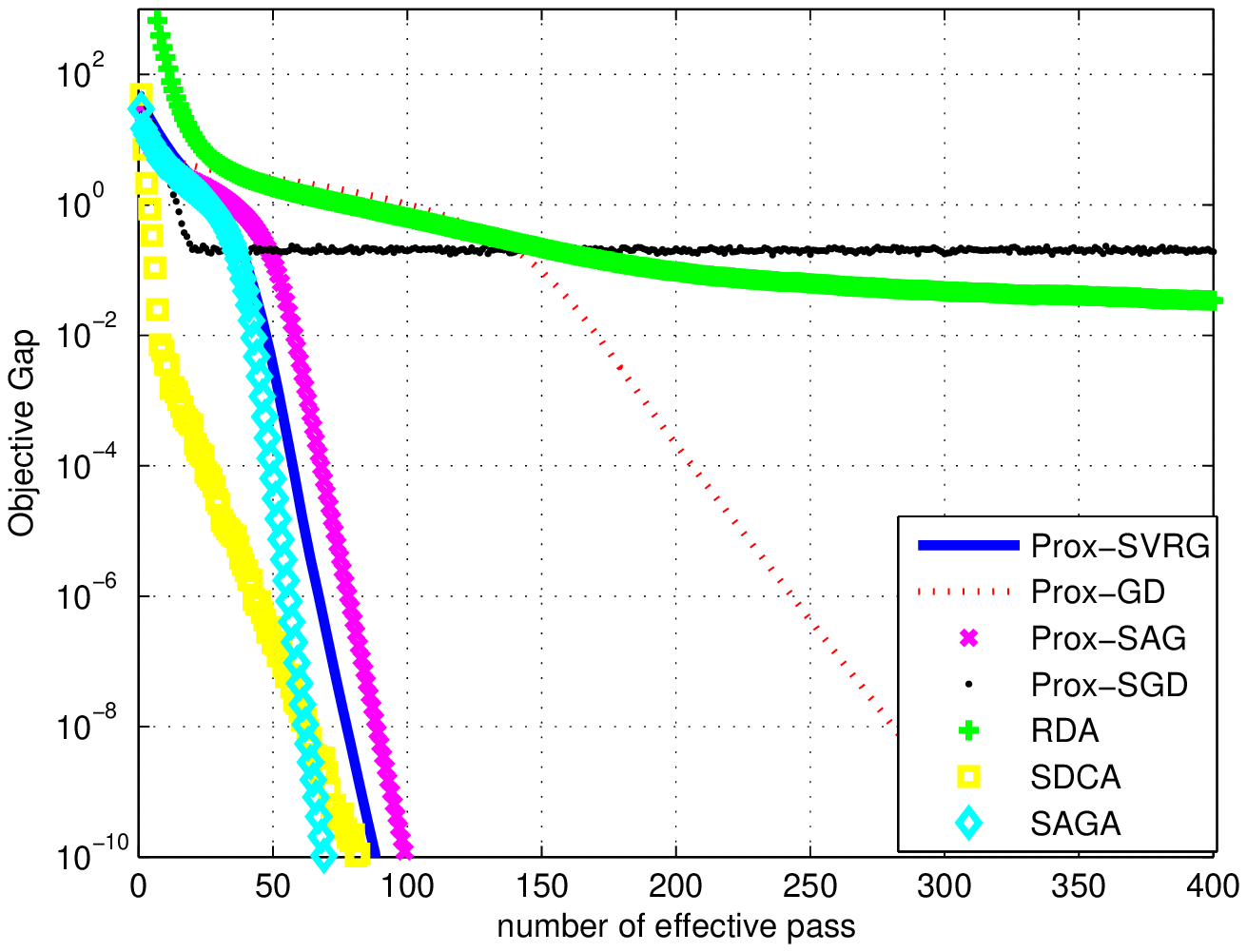}
			\caption{m=10, $s_{\mathcal{G}}=10, b=0$}
		\end{subfigure}
		\begin{subfigure}[b]{0.5\textwidth}
			\includegraphics[width=\textwidth]{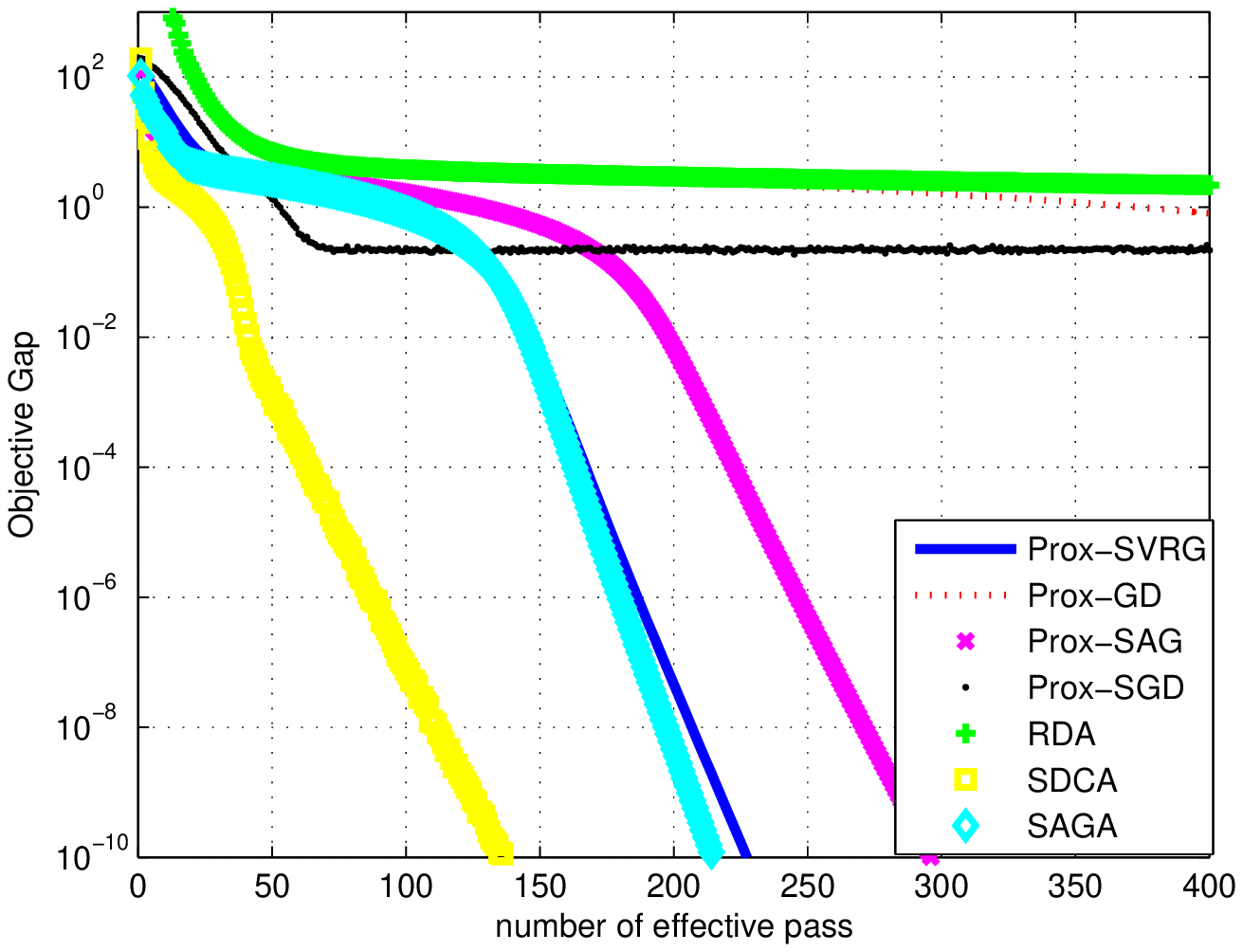}
			\caption{m=20, $s_{\mathcal{G}}=20, b=0$}
		\end{subfigure}
		\begin{subfigure}[b]{0.5\textwidth}
			\includegraphics[width=\textwidth]{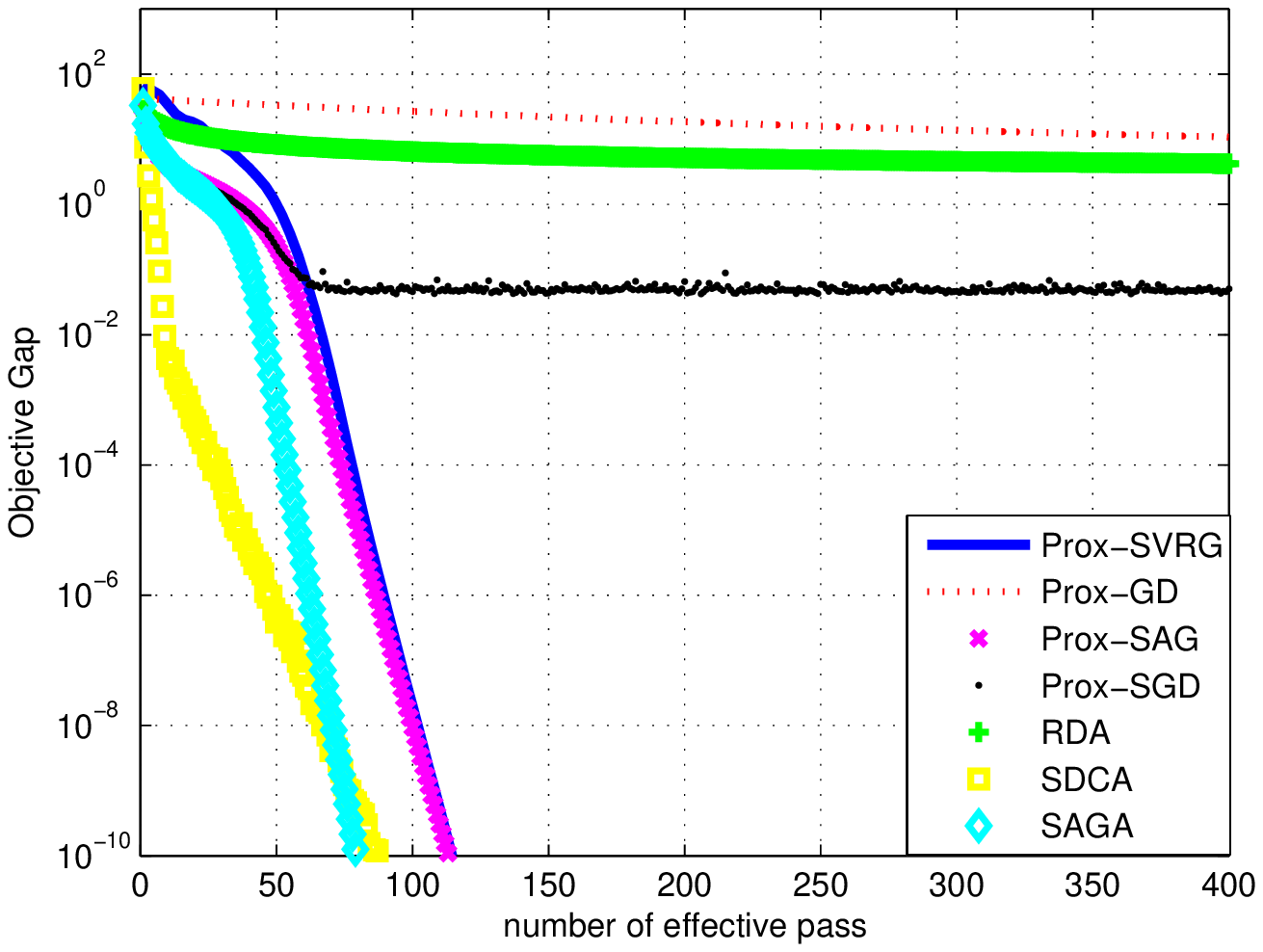}
			\caption{m=10, $s_{\mathcal{G}}=10, b=0.1$}
		\end{subfigure}
		\begin{subfigure}[b]{0.5\textwidth}
			\includegraphics[width=\textwidth]{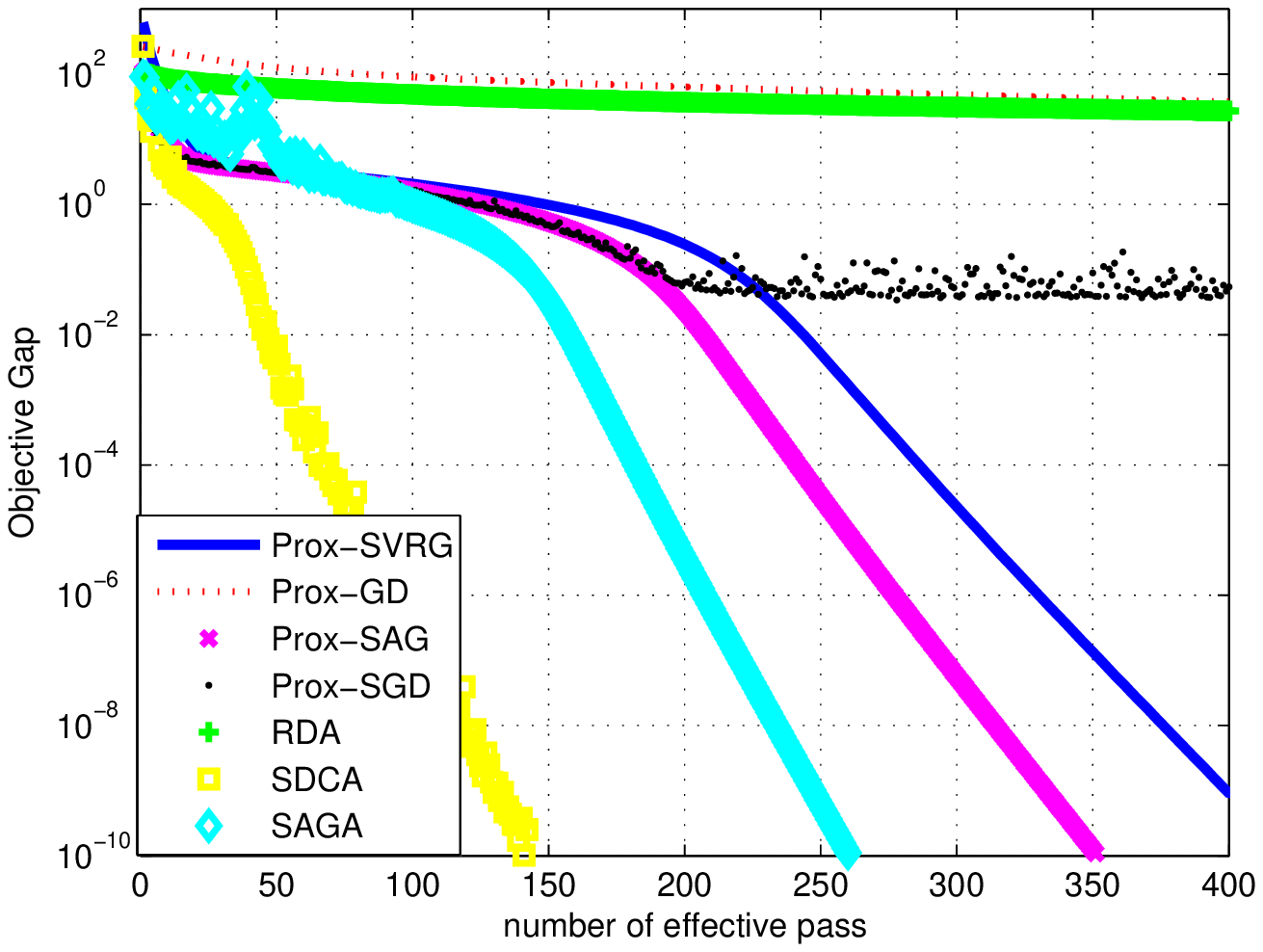}
			\caption{m=20, $s_{\mathcal{G}}=20, b=0.4$}
		\end{subfigure}
		
		\caption{Comparison between six algorithms on group Lasso. The x-axis is the number of pass over the dataset.  y-axis is the objective gap $F(w_k)-F(\hat{w})$ with a log scale.In Figure (a), m=10, $s_{\mathcal{G}}=10, b=0$. In Figure (b), m=20, $s_{\mathcal{G}}=20, b=0$. In Figure (c), m=10, $s_{\mathcal{G}}=10, b=0.1$. In Figure (d), m=20, $s_{\mathcal{G}}=20, b=0.4$. } \label{gp_lasso}
	\end{figure}
	
	Similar with the result in Lasso, SDCA, SVRG and SAG performs well in all settings. When $m=20$ and $s_{\mathcal{G}}=20$, SDCA outperforms the other two.  The Prox-GD converges with linear rate when $m=10, s_{G}=10, b=0$ but does not work well in the other three setting. SGD and RDA converge slowly in all four settings. 
	
	\subsubsection{Corrected Lasso}
	
	We generate data as follows: $ y_i=x_i^Tw^*+\xi_i$, where each data point $x_i\in \mathbb{R}^p$ is drawn from normal distribution $N(0,I)$, and the noise $ \xi_i$ is   drawn from $N(0,1)$. The coefficient $w^*$ is sparse with cardinality $s$, where the non-zero coefficient equals to $\pm 1$ generated from the Bernoulli distribution with probability $0.5$. We set covariance matrix $\Sigma_\varsigma=\gamma_\varsigma I$. We choose $\lambda=0.05$ in the formulation and $\tilde{\lambda}=0.1$ in SDCA. The result is presented in Figure \ref{fig:corrected_lasso}.  
	
	\begin{figure}[h]
		\begin{subfigure}[b]{0.45\textwidth}
			\centering
			\includegraphics[width=\textwidth]{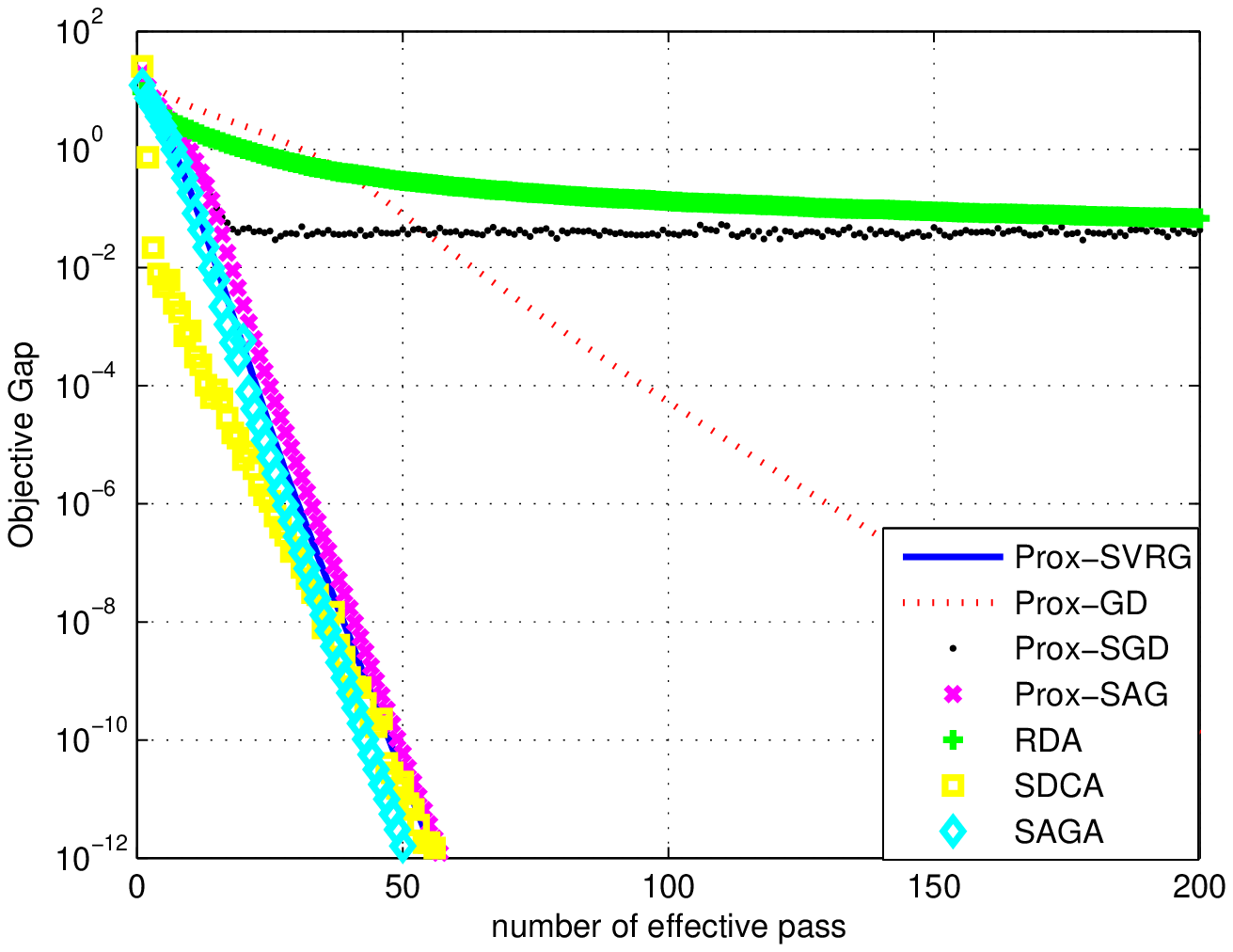}
			\caption{$n=2500, p=3000, s=50, \gamma_w=0.05$}
		\end{subfigure}
		\begin{subfigure}[b]{0.45\textwidth}
			\centering
			\includegraphics[width=\textwidth]{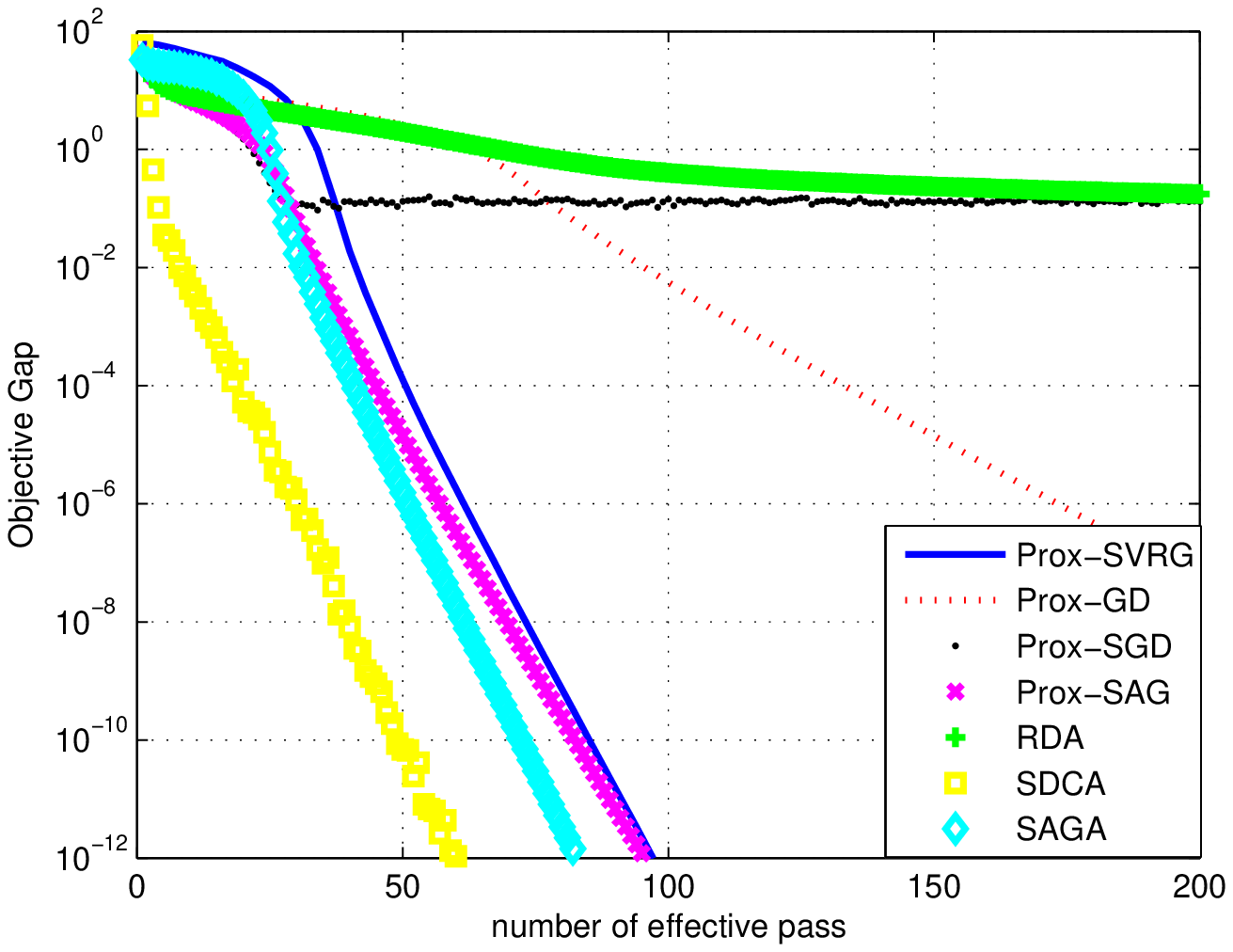}
			\caption{$n=2500,p=5000, s=100,\gamma_w=0.1$}
		\end{subfigure}
		\caption{Results on corrected Lasso. The x-axis is the number of pass over the dataset. y-axis is the objective gap $F(w_k)-F(\hat{w})$ with  log scale. We try two different settings. In the first figure $n=2500, p=3000, s=50, \gamma_w=0.05$. In the second figure $n=2500,p=5000, s=100,\gamma_w=0.1$.}\label{fig:corrected_lasso}
	\end{figure}

	In both setting, we see $\tilde{\kappa}> \tilde{\lambda}+\gamma_\varsigma$, thus according to our theory, SDCA converges linearly.  In both figures (a) and (b), SDCA, Prox-SVRG, Prox-SAG and Prox-GD converge linearly. SDCA performs better in the second setting.  SGD and RDA converge slowly due to the large variance in gradient.
	
	\subsubsection{SCAD}
	The way to generate data is same with Lasso. Here  $x_i\in \mathbb{R}^p$ is drawn from normal distribution $N(0,2I)$ to satisfy the requirement of on $\tilde{\kappa}$, $\gamma_\varsigma$ and $\tilde{\lambda} $.  We set $\lambda=0.05$ in the formulation and choose $\tilde{\lambda}=0.1$  in SDCA.   We present the result in Figure \ref{fig:SCAD}. We try two different settings on $n$, $p$, $s$, $\zeta$. In both case, SDCA, Prox-SVRG, Prox-SAG, converge linearly with similar performance. Prox-GD also converges linearly but with slower rate. According to our theory,  $\tilde{\kappa}\geq 1$ and $\tilde{\lambda}+\mu\leq 0.5$ in both cases, thus SDCA can converge linearly and the simulation results verify our theory.
	
	\begin{figure}[h]
		\begin{subfigure}[b]{0.45\textwidth}
			\centering
			\includegraphics[width=\textwidth]{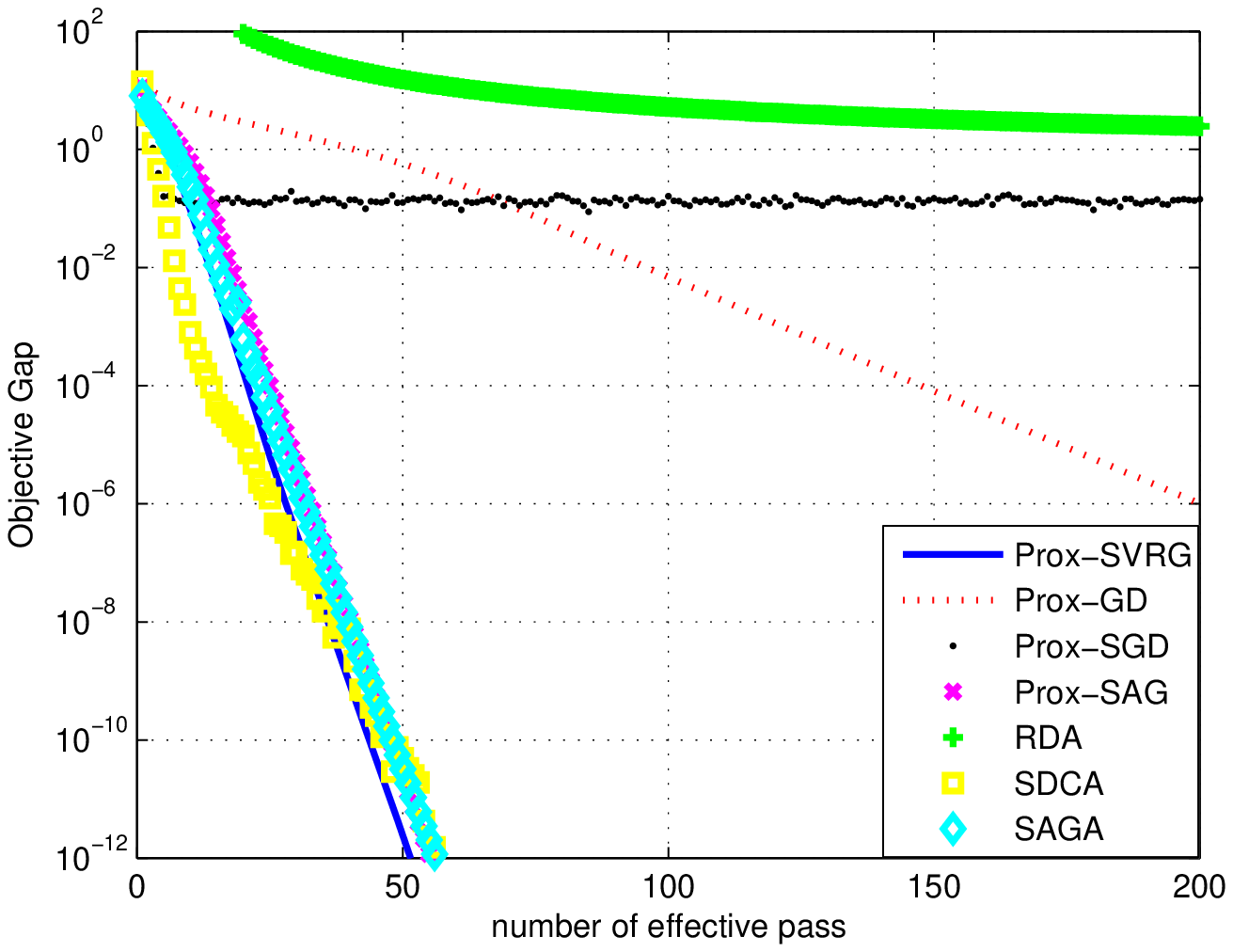}
			\caption{$n=3000, p=2500, s=30, \zeta=4.5 $ }
		\end{subfigure}
		\begin{subfigure}[b]{0.45\textwidth}
			\centering
			\includegraphics[width=\textwidth]{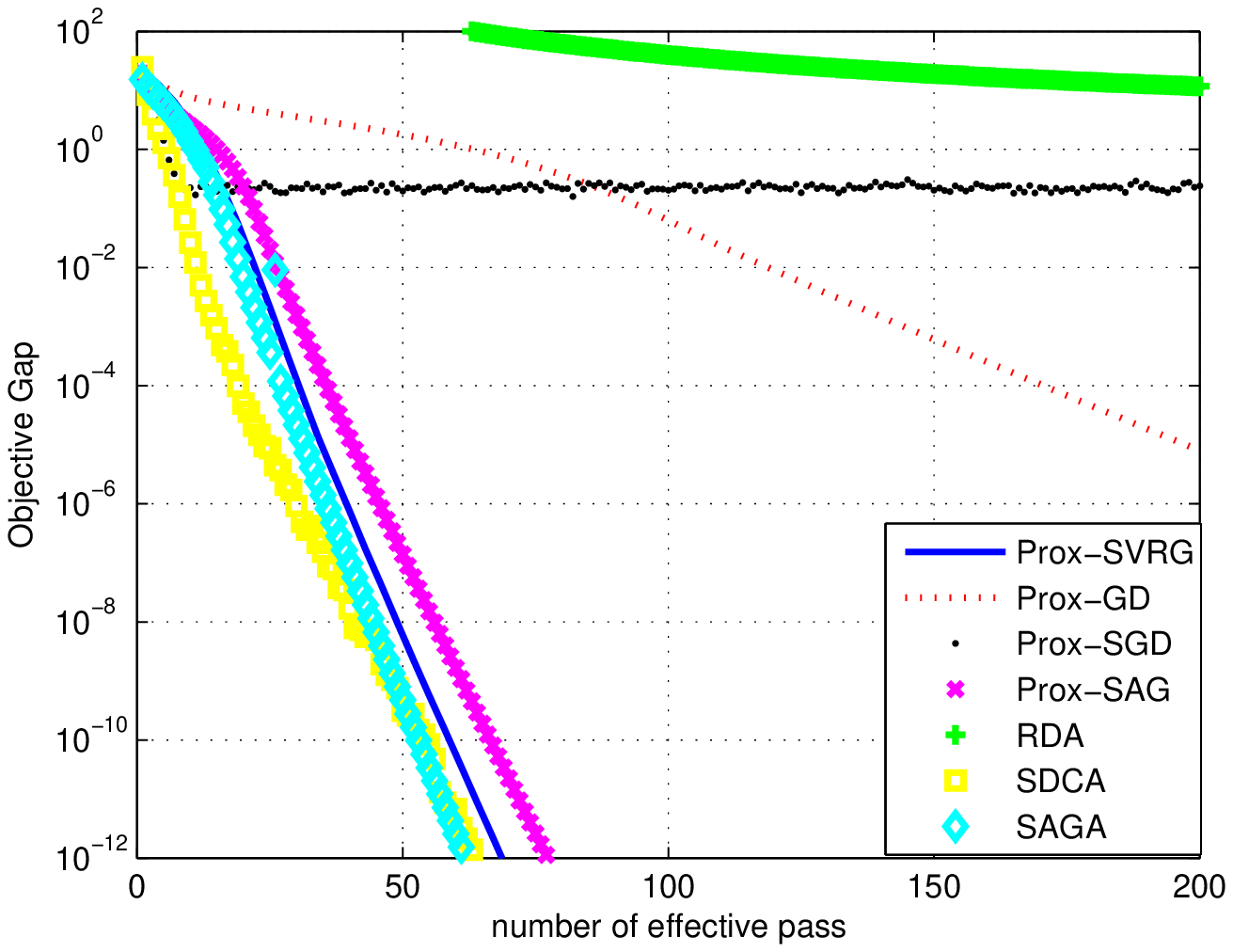}
			\caption{$n=2500, p=5000, s=50, \zeta=3.7 $  }
		\end{subfigure}
		\caption{Results on SCAD. The x-axis is the number of pass over the dataset. y-axis is the objective gap $F(w_k)-F(\hat{w})$ with  log scale. } \label{fig:SCAD}
	\end{figure} 
	
	\subsection{Real dataset}

	\subsubsection{Sparse Classification Problem }
	In this section, we evaluate the performance of the algorithms when solving the logistic regression with $\ell_1$ regularization: $ \min_{w} \sum_{i=1}^{n} \log (1+\exp (-y_ix_i^Tw))+\lambda \|w\|_1.$
	We conduct experiments on two real-world data sets, namely, rcv1 \cite{lewis2004rcv1} and sido0 \cite{SIDO}.  The regularization parameters are set as $\lambda=2\cdot 10^{-5}$ in rcv1 and $\lambda=10^{-4}$ in sido0, as suggested in \citet{xiao2014proximal}.  For SDCA, we choose $\tilde{\lambda}=0.002$ and $\tilde{\lambda}=0.001$ in these two experiments, respectively.

	\begin{figure}[h]
		\begin{subfigure}[b]{0.5\textwidth}
			\includegraphics[width=\textwidth]{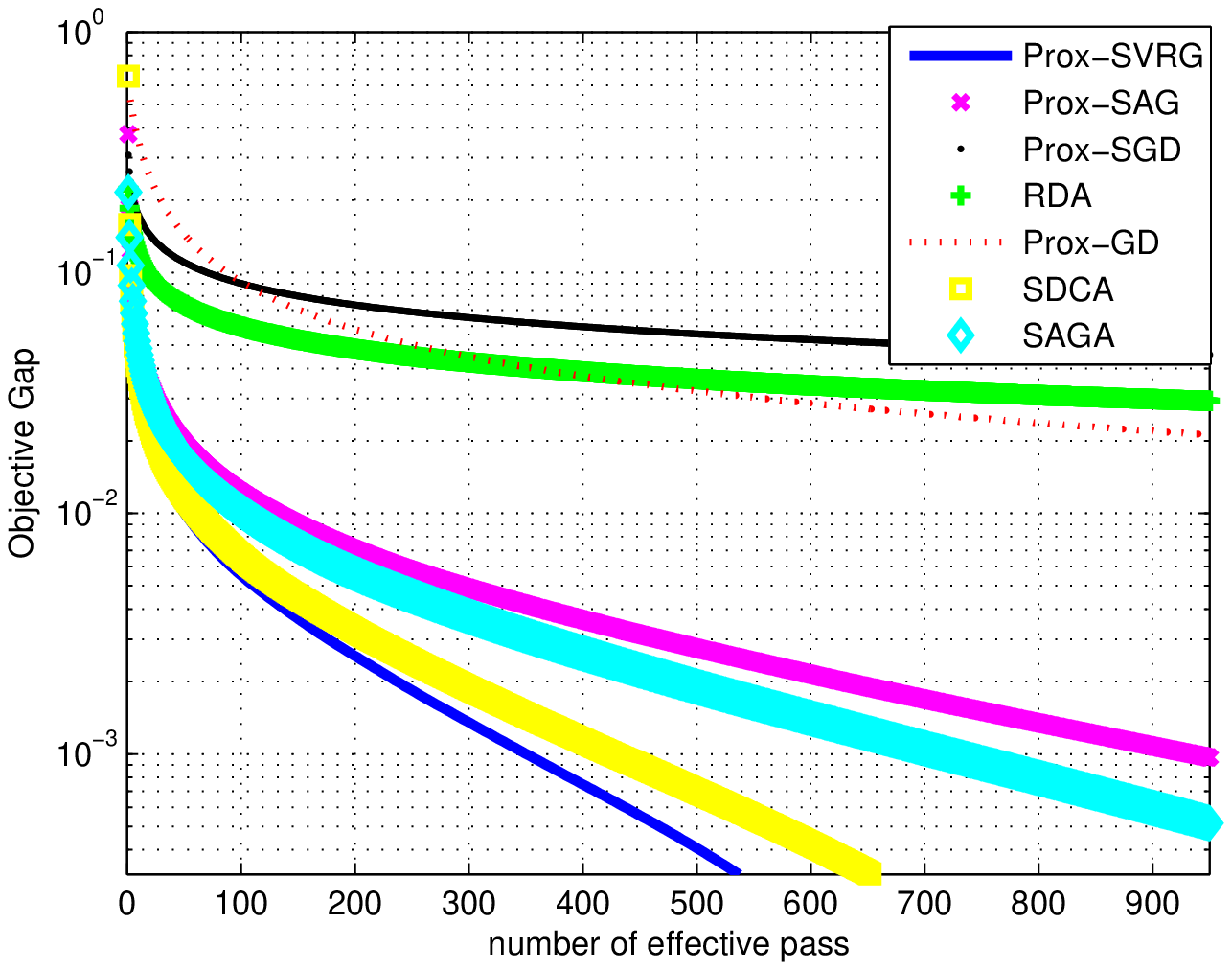}
			\caption{rcv1}\label{rcv1}
		\end{subfigure}
		\begin{subfigure}[b]{0.5\textwidth}
			\includegraphics[width=\textwidth]{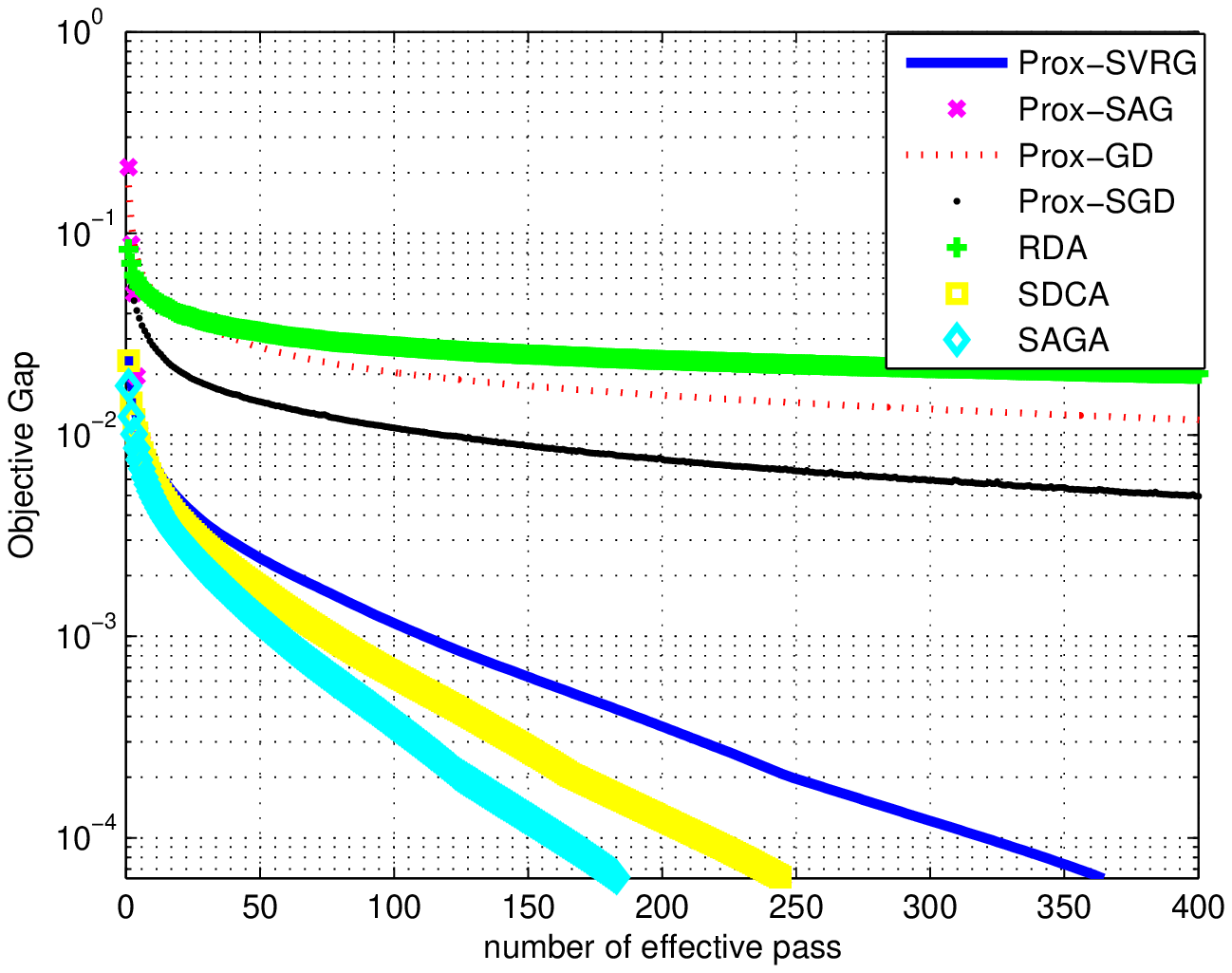}
			\caption{sido0}\label{sido0}
		\end{subfigure}
		\caption{The x-axis is the number of pass over the dataset, y-axis is the objective gap in the log-scale.}
	\end{figure}
	
	
	In Figure \ref{rcv1} and \ref{sido0} we report the performance of different algorithms.  In Figure \ref{rcv1} , Prox-SVRG performs best, and closely followed by SDCA, SAGA, and then Prox-SAG. We observe that Prox-GD converges much slower, albeit in theory it should converges with a linear rate \cite{agarwal2010fast}, possibly because its contraction factor is close to one. Prox-SGD and RDA converge slowly due to the variance in the stochastic gradient. The objective gaps of them remain significant even after 1000 passes of the whole dataset. In Figure \ref{sido0}, similarly as before, Prox-SVRG, SAGA, SDCA and Prox-SAG (some part of Prox-SAG  overlaps with SDCA)  perform well. On this dataset, SAGA performs best followed by SDCA , Prox-SAG and then Prox-SVRG. The performance of  Prox-GD is even worse   than   Prox-SGD.   RDA converges the slowest.
	
	\subsubsection{Sparse Regression Problem }
	
	In Figure \ref{IJCNN1}, we present the result of Lasso on IJCNN1 dataset ($n=49990,p=22$)  \cite{IJCNN1} with $\lambda$=0.02. It is easy to see, the performance of SDCA is best and then followed by SAGA, Prox-SAG, SVRG and Prox-GD. Prox-SGD and RDA does not work well. 
		We apply the linear regression with SCAD regularization on IJCNN1 dataset \cite{IJCNN1} and present the result in Figure \ref{IJCNN_SCAD}. In this dataset, SAGA and SDCA have almost identical performance, then followed by Prox-SVRG, Prox-SAG and Prox-GD. Prox-SGD converge fast at beginning, but has a large optimality gap ($10^{-4}$). RDA does not work at all.
		  In Figure \ref{boston_housing},  we consider a group sparse regression problem on the Boston Housing dataset ($n=506,p=13$)  \cite{uci:2013}.  As suggested in \citet{swirszcz2009grouped,xiang2014simultaneous}, to take into account the non-linear relationship between variables and response, up to third-degree polynomial expansion is applied on each feature. In particular, terms $x$, $x^2$ and $x^3$ are grouped together. We consider group Lasso model on this problem with $\lambda=0.1$. We choose the setting $m=2n$ in SVRG and $\tilde{\lambda}=0.1$ in SDCA. 	Figure \ref{boston_housing} shows the objective gap of various algorithms versus the number of passes over the dataset. Prox-SVRG, SDCA, SAGA and Prox-SAG have almost same performance. Prox-SGD does not converge: the objective gap oscillates between $0.1$ and $1$. Both the Prox-GD and RDA converge, but with much slower rates.

	\begin{figure}
		\begin{subfigure}[b]{0.3\textwidth}
			\includegraphics[width=\textwidth]{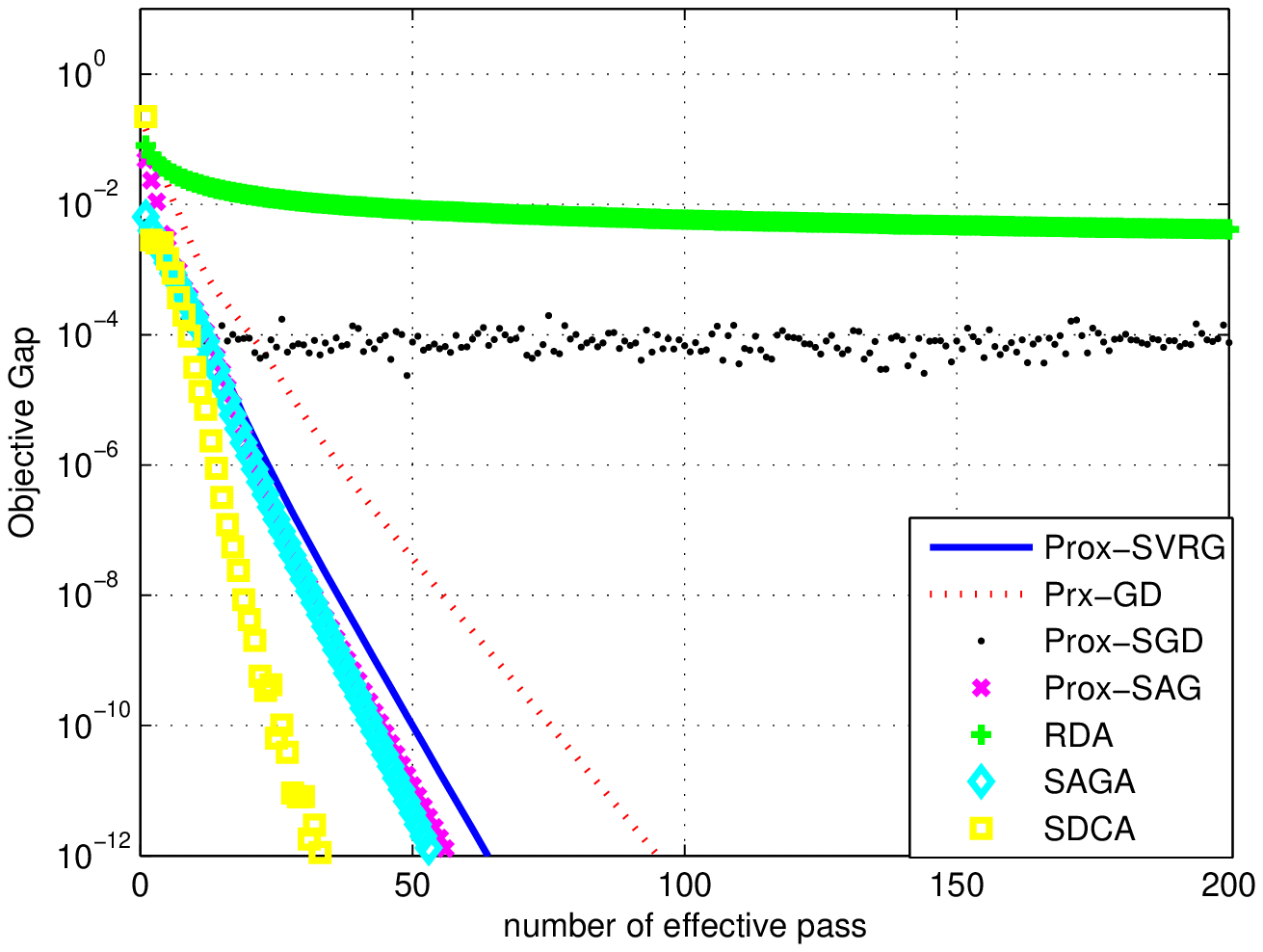}
			\caption{Lasso}\label{IJCNN1}
		\end{subfigure}
	  \begin{subfigure}[b]{0.3\textwidth}
				\includegraphics[width=\textwidth]{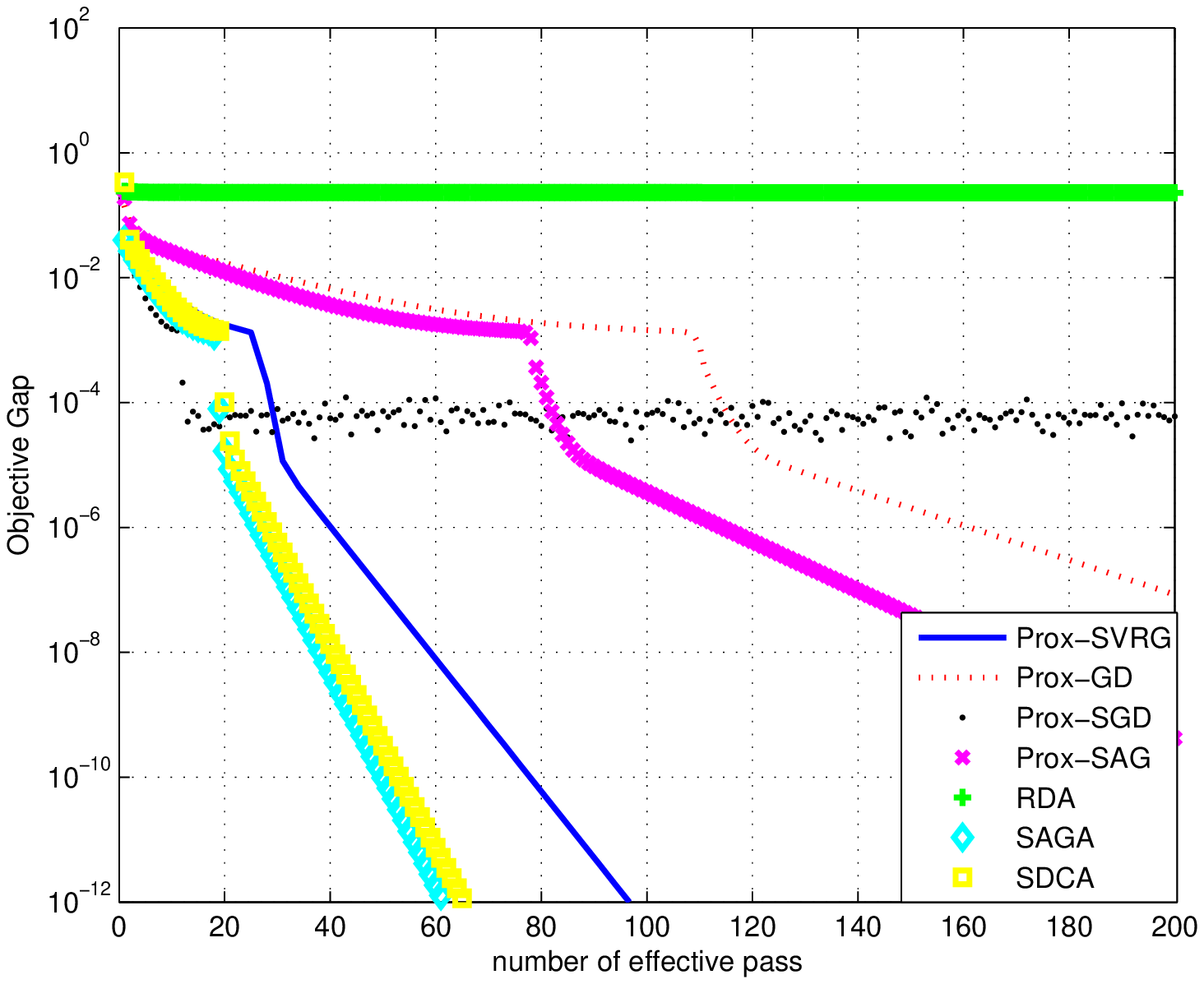}
				\caption{ SCAD }\label{IJCNN_SCAD}
	  \end{subfigure}
		\begin{subfigure}[b]{0.3\textwidth}
			\includegraphics[width=\textwidth]{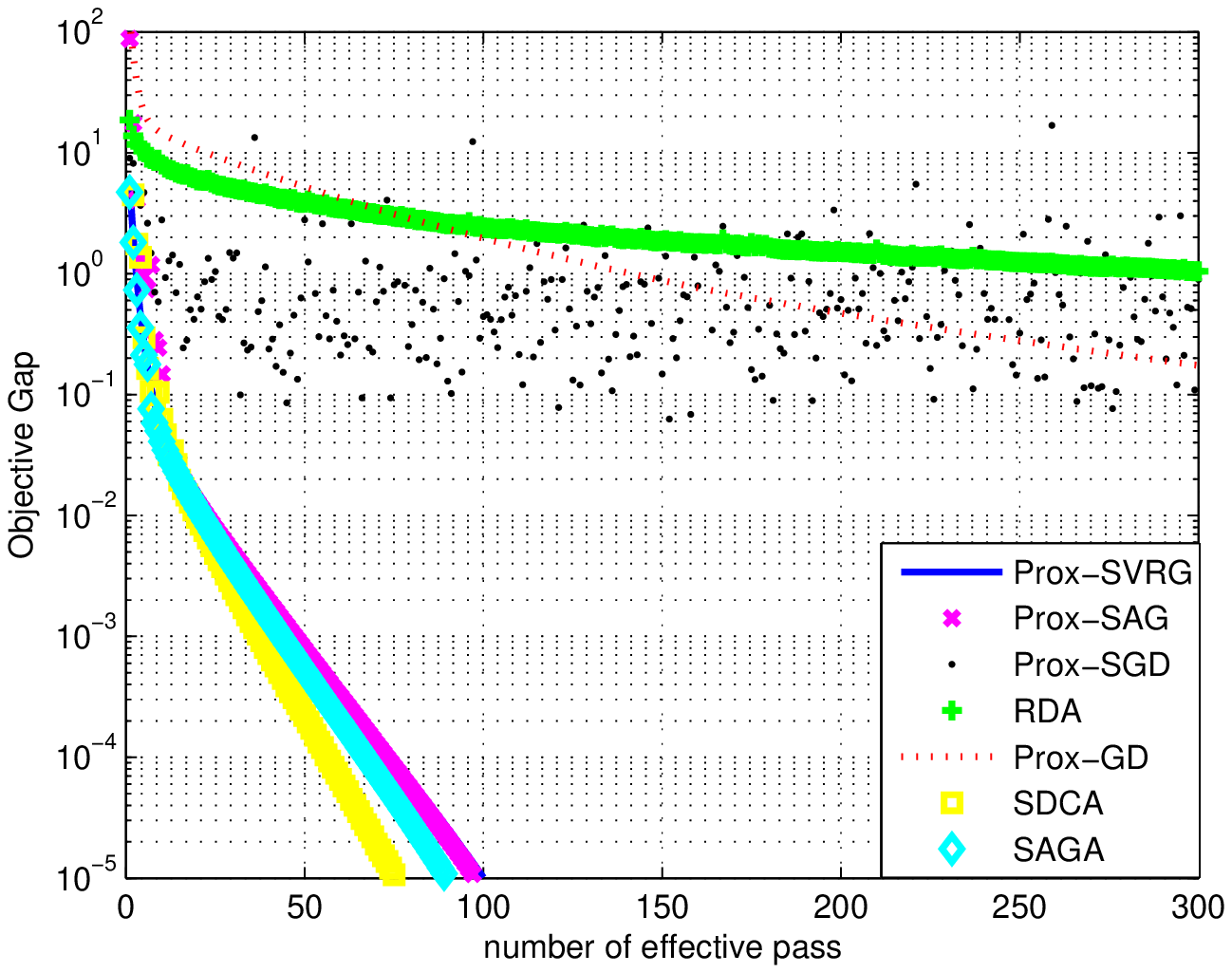}
			\caption{Group Lasso}\label{boston_housing}
		\end{subfigure}
		
		\caption{The x-axis is the number of pass over the dataset, y-axis is the objective gap in the log-scale.}
	\end{figure}

	\section{Conclusion and future work}
	
	In this paper, we adapt SDCA into a dual-free form to solve  non-strongly convex problems  and non-convex problems. Under the condition of RSC, we prove that this dual-free SDCA　 converges with linear rates which covers several important statistical models. 
	From a high level, our results re-confirmed a well-observed  phenomenon that statistically easy problems   tend to be more computationally friendly. We believe this  intriguing fact indicates that there are fundamental relationships between statistics and optimization,  and understanding such relationships may shed deep insights.

\appendix	
	
	\section{Proofs}\label{app.proof}
	
	To begin with, we present a technical Lemma that will be used repeatedly in the following proofs.
	\begin{lemma}\label{Lemma.smooth}
		Suppose function f(x) is convex and L smooth then we have
		$$ f(x)-f(y)-\langle \nabla f(y),x-y \rangle\geq \frac{1}{2L} \|\nabla f(x)-\nabla f(y)\|_2^2.$$
	\end{lemma}
	\begin{proof}
		Define $p(x)=f(x)-\langle \nabla f(x_0),x\rangle $. It is obvious that $p(x_0)$ is the optimal value of $p(x)$  due to convexity. 
		Since $f(x)$ is $L$ smooth, so is $p(x)$, and we have
		$$ p(x_0)\leq p\left(x-\frac{1}{L}\nabla p(x)\right)\leq p(x)-\frac{1}{2L} \|\nabla p(x)\|_2^2.  $$
		That is	
		$$f(x_0)-\langle \nabla f(x_0),x_0 \rangle\leq f(x)-\langle \nabla f(x_0),x \rangle-\frac{1}{2L}\|\nabla f(x)-\nabla f(x_0)\|_2^2.$$	
		Rearrange the terms, we have
		$$ f(x)-f(x_0)-\langle f(x_0),x-x_0 \rangle\geq \frac{1}{2L}\|\nabla f(x)-\nabla f(x_0)\|_2^2.$$
	\end{proof}
	The following Lemma presents a well-known fact on conjugate function of a strongly convex function. We extract it from Theorem 1 of \cite{nesterov2005smooth}.
	\begin{lemma}\label{Lemma.conjugate}
		Define $p^*(v)=\max_{w\in \Omega} \langle w,v\rangle-p(w)$, where $\Omega$ is a convex compact set, $p(w)$ is a 1-strongly convex function, then $p^*(v)$ is $1$-smooth and $\nabla p^*(v)=\bar{w}$ where $\bar{w}=\arg\max_{w\in \Omega} \langle w,v\rangle-p(w)$.
	\end{lemma}
	\subsection{Proof of results for Convex $F(w)$}
	In this section we establish the results for convex $F(w)$, namely, Theorem \ref{main_theorem}. 
	Recall the problem we want to optimize is
	\begin{equation} \label{obj}
	\min_{w\in\Omega} F(w):=f(w)+\lambda 
	g(w):=\frac{1}{n}\sum_{i=1}^{n} f_i (w)+\lambda g(w),
	\end{equation}
	where $\Omega$ is $ \{w\in \mathbb{R}^p  | g(w)\leq \rho\}$.
	Remind that instead of directly minimizing the above objective function, we look at the following form.
	$$ \min_{w\in\Omega} F(w):=\phi(w)+\tilde{\lambda}\tilde{g}(w)= \frac{1}{1+n} \sum_{i=1}^{n+1}\phi_i(w)+ \tilde{\lambda} \tilde{g}(w).$$ 
	Also remind that $\tilde{L}_i$ is the smooth parameter of $\phi_i(w) $ and we define $ \tilde{L}\triangleq\frac{1}{n+1} \sum_{i=1}^{n+1} \tilde{L}_i.$

	Notice the Algorithm 1  keeps the  following relation $v^{t-1}=\frac{1}{\tilde{\lambda} (n+1)} \sum_{i=1}^{n+1}a_i^{t-1}$, and thus we have: 
	\begin{equation}\label{unbias_estimator}
	\begin{split}
	E (\eta_i  (\nabla \phi_i(w^{t-1})+a_i^{t-1}))&=\frac{\eta}{n+1} \sum_{i=1}^{n+1} (\nabla \phi_i(w^{t-1})+a_i^{t-1})\\
	&=\eta (\frac{1}{n+1}\sum_{i=1}^{n+1} \nabla\phi_i(w^{t-1})+\tilde{\lambda} v^{t-1})\\
	& =\eta (\nabla \phi(w^{t-1})+\tilde{\lambda} v^{t-1}).
	\end{split}
	\end{equation}

	Before we start the proof of the main theorem , we present several technical lemmas.  The following two lemmas are similar to its batched counterpart in \cite{agarwal2010fast}.
	\begin{lemma}\label{lemma.cone}
		Suppose $f(w)$ is convex and $g(w)$ is decomposable with respect to $(\mathcal{A},\mathcal{B})$, $g(w^*)\leq \rho$, if we choose $\lambda\geq 2g^*(\nabla f(w^*))$, define the error term $\Delta^*=\hat{w}-w^*$, then we have the following condition holds 
		$$ g(\Delta^*_{\mathcal{B}^\perp})\leq 3 g(\Delta^*_{\mathcal{B}})+4g(w^*_{\mathcal{A}^\perp}), $$
		which implies $ g (\Delta^*)\leq g(\Delta^*_{\mathcal{B}^\perp})+g(\Delta^*_{\mathcal{B}})\leq 4 g(\Delta^*_{\mathcal{B}})+4g(w^*_{\mathcal{A}^\perp}).$
	\end{lemma}

	\begin{proof}
		Using the optimality of $\hat{w}$, we have
		
		$$ f(\hat{w})+\lambda g(\hat w)-f(w^*)-\lambda g(w^*)\leq 0.$$
		So we get
		$$ \lambda g(w^*)-\lambda g(\hat{w})\geq f(\hat{w})-f(w^*)\geq \langle \nabla f(w^*), \hat{w}-w^*  \rangle\geq -g^*(\nabla f(w ^*)) g (\Delta^*), $$
		where the second inequality holds from the convexity of $f(w)$, and the third one holds by Holder's inequality. Using triangle inequality, we have
		$ g(\Delta^*)\leq g(\Delta^*_{\mathcal{B}})+g(\Delta^*_{\mathcal{B}^{\perp}}),$ which leads to
		\begin{equation}\label{eq:lemma.cone}
		\lambda g(w^*)-\lambda g(\hat{w})\geq -g^*(\nabla f(w ^*)) (g(\Delta^*_{\mathcal{B}})+g(\Delta^*_{\mathcal{B}^{\perp}}) ).
		\end{equation}
		Notice $$\hat{w}=w^*+\Delta^*= w^*_{\mathcal{A}}+w^*_{\mathcal{A}^{\perp}}+\Delta^*_{\mathcal{B}}+\Delta^*_{\mathcal{B}^{\perp}}. $$
		Now we obtain
		\begin{equation}
		\begin{split}
		g  (\hat{w})-g(w^*)& \overset{(a)}{\geq}  g (w^*_{\mathcal{A}}+\Delta^*_{\mathcal{B}^{\perp}})-g(w^*_{\mathcal{A}^{\perp}})-g(\Delta^*_{\mathcal{B}})-g(w^*)\\
		& \overset{(b)}{= }g (w^*_{\mathcal{A}})+g (\Delta^*_{\mathcal{B}^{\perp}})-g(w^*_{\mathcal{A}^{\perp}})-g(\Delta^*_{\mathcal{B}})-g(w^*)\\
		& \overset{(c)}{\geq} g (w^*_{\mathcal{A}})+g (\Delta^*_{\mathcal{B}^{\perp}})-g(w^*_{\mathcal{A}^{\perp}})-g(\Delta^*_{\mathcal{B}})-g(w^*_{\mathcal{A}})-g(w^*_{\mathcal{A}^{\perp}})\\
		&\geq g (\Delta^*_{\mathcal{B}^{\perp}})-2g(w^*_{\mathcal{A}^{\perp}})-g(\Delta^*_{\mathcal{B}}),
		\end{split}
		\end{equation}
		where $(a)$ and $(c)$ hold from the triangle inequality, and $(b)$ uses the decomposability of $g(\cdot)$.

		Substitute  the above to  \eqref{eq:lemma.cone}, and  and use the assumption that $\lambda\geq 2g^*(\nabla f(w^*))$, we obtain
		$$-\frac{\lambda}{2} (g(\Delta^*_{\mathcal{B}})+g(\Delta^*_{\mathcal{B}^{\perp}}) )+\lambda (g (\Delta^*_{\mathcal{B}^{\perp}})-2g(w^*_{\mathcal{A}^{\perp}})-g(\Delta^*_{\mathcal{B}}))\leq 0,$$
		which implies 
		$$ g(\Delta^*_{\mathcal{B}^\perp})\leq 3  g(\Delta^*_{\mathcal{B}})+4g(w^*_{\mathcal{A}^{\perp}}).$$
	\end{proof}

	\begin{lemma}\label{lemma.cone_optimization}
		Suppose	$f(w)$ is convex and $g(w)$ is decomposable with respect to $(\mathcal{A},\mathcal{B})$. If we choose $\lambda\geq 2g^*(\nabla f(w^*))$,  and suppose there exist a given tolerance $\xi$ and $T$ such that $ F(w^t)-F(\hat{w})\leq \xi$ for all $t> T$ ,  then for the error term $\Delta^t=w^t-w^* $ we have
		$$ g(\Delta^t_{\mathcal{B}^\perp})\leq 3 g(\Delta^t_{\mathcal{B}})+4g(w^*_{\mathcal{A}^\perp})+ 2\min \{\frac{\xi}{\lambda}, \rho\}, $$
		which implies $$ g(\Delta^t)\leq 4 g(\Delta^t_{\mathcal{B}})+4g(w^*_{\mathcal{A}^\perp})+ 2 \min \{\frac{\xi}{\lambda},\rho\}. $$
	\end{lemma}
	\begin{proof}
		First notice $ F(w^t)-F(w^*)\leq \xi $ holds by assumption since $F(w^*)\geq F(\hat{w}).$
		So we have
		$$ f(w^t)+\lambda g(w^t)-f(w^*)-\lambda g(w^*)\leq \xi.$$
		Follow the same steps as those in the proof of Lemma \ref{lemma.cone}, we have
		$$ g(\Delta^t_{\mathcal{B}^\perp})\leq 3 g(\Delta^t_{\mathcal{B}})+4g(w^*_{\mathcal{A}^\perp})+ 2\frac{\xi}{\lambda}. $$
		Using the fact that $w^*$ and $w^t$ both belong to $\Omega=\{w\in \mathbb{R}^p  | g(w)\leq \rho\}$,    we have  $g(\Delta^t)\leq g (w^*)+g(w^t)\leq 2\rho $. This leads to
		$$ g(\Delta^t_{\mathcal{B}^\perp})\leq g(\Delta^t_{\mathcal{B}})+2\rho,
		$$
		by triangle inequality $g(\Delta^t_{\mathcal{B}^\perp})\leq g(\Delta^t_{\mathcal{B}})+g(\Delta^t) $.
		Combine this with the above result we have
		$$ g(\Delta^t_{\mathcal{B}^\perp})\leq 3 g(\Delta^t_{\mathcal{B}})+4g(w^*_{\mathcal{A}^\perp})+ 2\min\{\frac{\xi}{\lambda},\rho\}. $$
		The second statement follows immediately from $g(\Delta^t)\leq g(\Delta^t_{\mathcal{B}})+g(\Delta^t_{\mathcal{B^\perp}}).$	
	\end{proof}

	\begin{lemma}\label{lemma.RSC_cone}
		Under the same assumption of Lemma \ref{lemma.cone_optimization}, we have
		\begin{equation}\label{RSC_cone}
		F(w^t)-F(\hat{w})\geq \left(\frac{\kappa}{2}-32\tau \Psi^2(\mathcal{B} )\right)\|\hat{\Delta}^t \|_2^2- \epsilon^2(\Delta^*,\mathcal{A},\mathcal{B}),
		\end{equation}
		and 
		\begin{equation}\label{equ.xu.lemma5}
		\begin{split}
		&\phi(\hat{w})-\phi(w^{t})-\langle \hat{w}-w^{t},\nabla \phi(w^{t})\rangle\\
		&\geq  \left[\left(\frac{\kappa-\tilde{\lambda}}{2}-32\tau \Psi^2(\mathcal{B} )\right)\|\hat{\Delta}^t \|_2^2- \epsilon^2(\Delta^*,\mathcal{A},\mathcal{B})\right],
		\end{split}
		\end{equation}	
		where $\hat{\Delta}^t=w^t-\hat{w}$, $ \epsilon^2(\Delta^*,\mathcal{A},\mathcal{B})=2\tau (\delta_{stat}+\delta)^2 $, $\delta=2\min \{\frac{\xi}{\lambda}, \rho\} $, and $\delta_{stat}= 8\Psi(\mathcal{B})\|\Delta^*\|_2+8g(w^*_{\mathcal{A}^\perp}) $.
	\end{lemma} 
	
	\begin{proof}
		We begin the proof by establishing a simple fact on $\hat{\Delta}^t=w^t-\hat{w}$. We adapt the argument in Lemma \ref{lemma.cone_optimization} (which is on  $\Delta^t$)  to $\hat{\Delta}^t$:
		\begin{equation}
		\begin{split}
		g(\hat{\Delta}^t)&\leq g (\Delta^t)+g(\Delta^*)\\
		&\leq  4 g(\Delta^t_{\mathcal{B}})+4g(w^*_{\mathcal{A}^\perp})+ 2\min \{\frac{\xi}{\lambda}, \rho\}+4 g(\Delta^*_{\mathcal{B}})+4g(w^*_{\mathcal{A}^\perp})\\
		& \leq 4 \Psi(\mathcal{B})\|\Delta^t\|_2+4\Psi(\mathcal{B})\|\Delta^*\|_2+8g(w^*_{\mathcal{A}^\perp})+2\min \{\frac{\xi}{\lambda},\rho\},
		\end{split}
		\end{equation}
		where the first inequality holds from the triangle inequality, the second inequality uses Lemma \ref{lemma.cone} and \ref{lemma.cone_optimization}, the third holds because of the definition of subspace compatibility.

		We know $$ f(w^t)-f(\hat{w})-\langle \nabla f(\hat{w}), \hat{\Delta}^t  \rangle\geq \frac{\kappa}{2} \|\hat{\Delta}^t \|_2^2-\tau g^2(\hat{\Delta}^t) $$
		which implies $ F(w^t)-F(\hat{w})\geq \frac{\kappa}{2} \|\hat{\Delta}^t \|_2^2-\tau g^2(\hat{\Delta}^t)$, since $\hat{w}$ is the optimal solution to the problem \ref{obj} and $g(w)$ is convex.
		Notice that
		\begin{equation}
		\begin{split}
		g(\hat{\Delta}^t)&\leq 4 \Psi(\mathcal{B})\|\Delta^t\|_2+4\Psi(\mathcal{B})\|\Delta^*\|_2+8g(w^*_{\mathcal{A}^\perp})+2\min \{\frac{\xi}{\lambda},\rho\}\\
		&\leq  4 \Psi(\mathcal{B})\|\hat{\Delta}^t\|_2+8\Psi(\mathcal{B})\|\Delta^*\|_2+8g(w^*_{\mathcal{A}^\perp})+2\min \{\frac{\xi}{\lambda},\rho\},
		\end{split}
		\end{equation}
		where the second inequality uses the triangle inequality.
		Using the inequality $ (a+b)^2\leq 2 a^2+2b^2 $, we can upper bound $g^2 (\hat{\Delta}^t)$.
		\begin{equation}\label{eq:lemma.RSC_cone}
		g^2 (\hat{\Delta}^t)\leq 32\Psi^2(\mathcal{B})\|\hat{\Delta}^t\|_2^2+2[8\Psi(\mathcal{B})\|\Delta^*\|_2+8g(w^*_{\mathcal{A}^\perp})+2\min \{\frac{\xi}{\lambda},\rho\}]^2.
		\end{equation}
		We now use above result to rewrite the RSC condition.

		Substitute this upper bound in the RSC, we have
		$$ F(w^t)-F(\hat{w})\geq \left(\frac{\kappa}{2}-32\tau \Psi^2(\mathcal{B} )\right)\|\hat{\Delta}^t \|_2^2-2\tau[8\Psi(\mathcal{B})\|\Delta^*\|_2+8g(w^*_{\mathcal{A}^\perp})+2\min \{\frac{\xi}{\lambda},\rho\}]^2. $$
		Notice by $\delta=2\min \{\frac{\xi}{\lambda},\rho\} $, $\delta_{stat}= 8\Psi(\mathcal{B})\|\Delta^*\|_2+8g(w^*_{\mathcal{A}^\perp}),$  and	$ \epsilon^2(\Delta^*,\mathcal{A},\mathcal{B})=2\tau(\delta_{stat}+\delta)^2 $,  we obtain
		$$\epsilon^2(\Delta^*,\mathcal{A},\mathcal{B})=2\tau(8\Psi(\mathcal{B})\|\Delta^*\|_2+8g(w^*_{\mathcal{A}^\perp})+2\min \{\frac{\xi}{\lambda},\rho\})^2.$$
		We thus conclude
		\begin{equation}\label{RSC_modified}
		F(w^t)-F(\hat{w})\geq \left(\frac{\kappa}{2}-32\tau \Psi^2(\mathcal{B} )\right)\|\hat{\Delta}^t \|_2^2- \epsilon^2(\Delta^*,\mathcal{A},\mathcal{B}).
		\end{equation}
		Recall $\phi(w)= f(w)-\frac{\tilde{\lambda}}{2}\|w\|_2^2
		$, and hence we have
		\begin{equation}
		\begin{split}
		&\phi(\hat{w})-\phi(w^{t})-\langle \hat{w}-w^{t},\nabla \phi(w^{t})\rangle\\
		=& f(\hat{w})-f(w^{t})-\langle\nabla f(w^{t}), \hat{w}-w^{t}\rangle-\frac{\tilde{\lambda}}{2}\|\hat{w}-w^{t}\|_2^2 \\
		\geq& \frac{\kappa}{2} \|\hat{\Delta}^t\|_2^2-\tau g^2 (\hat{\Delta}^t)-\frac{\tilde{\lambda}}{2}\|\hat{\Delta}^t\|_2^2,
		\end{split}
		\end{equation}
		where the inequality is due to the RSC condition.
		Now we plug in  the upper bound of $g^2(\hat{\Delta}^t)$ in Equation~ \eqref{eq:lemma.RSC_cone}, and arrange the terms to establish Equation~\eqref{equ.xu.lemma5}.	
	\end{proof}

	\begin{proof}[Proof of Theorem on convex $F(w)$, i.e., Theorem \ref{main_theorem}]
		Recall we define two potentials
		\begin{equation*}\begin{split}
		A_t&=\sum_{j=1}^{n+1}\frac{1}{q_i}\|a_j^t-\hat{a}_j\|_2^2, \\
		B_t &=2(\tilde{g}^*(v^t)-\langle \nabla \tilde{g}^*(\hat{v}),v^t-\hat{v} \rangle-\tilde{g}^*(\hat{v})).
		\end{split}\end{equation*}
		Notice using   Theorem 1 in \cite{nesterov2005smooth},  we know  $ \tilde{g}^* (v)$ is $1-$ smooth and $w=\nabla \tilde{g}^* (v)$.

		We remark that the potential $A_t$ is defined the same as in \cite{shalev2016sdca} while we define $B_t$   differently to solve the problem with general regularization $g(w)$. When $ \tilde {g}(w)=\frac{1}{2}\|w\|_2^2$, $B_t$ reduce to $ \|v^t-\hat{v}\|_2^2=\|w^t-\hat{w}\|_2^2$, which is same as  in \cite{shalev2016sdca}.

		\textbf{ Step 1.} The first step is to lower bound $A_{t-1}-A_t$, in particular, to establish that 
		\begin{equation}\label{equ.xu.thm1-step1}
		\mathbb{E}[A_{t-1}-A_t]=\eta \tilde{\lambda} \left( A_{t-1}+\sum_{i=1}^{n+1}\frac{1}{q_i} (-\|u_i-\hat{a}_i\|_2^2+(1-\beta_i)\|m_i\|_2^2)  \right).
		\end{equation}
		This step is indeed same as \cite{shalev2016sdca}, which we present for the completeness. Define $u_i=-\nabla \phi_i(w^{t-1})$, $\beta_i=\eta_i\tilde{\lambda} (n+1)$ and $m_i=-u_i+a_i^{t-1}$ for notational simplicity. Suppose coordinate $i$ is picked up at time $t$, then we have 
		\begin{equation}
		\begin{split}
		A_{t-1}-A_t=& -\frac{1}{q_i}\|a_i^t-\hat{a}_i\|_2^2+\frac{1}{q_i}\|a_i^{t-1}-\hat{a}_i\|_2^2\\
		&=-\frac{1}{q_i} \|(1-\beta_i) (a_{i}^{t-1}-\hat{a}_i)+\beta_i(u_i-\hat{a}_i) \|_2^2+\frac{1}{q_i}\|a_i^{t-1}-\hat{a}_i\|_2^2\\
		&\geq -\frac{1}{q_i} \left(  (1-\beta_i)\|a_i^{t-1}-\hat{a}_i \|_2^2+\beta_i \|u_i-\hat{a}_i\|_2^2-\beta_i (1-\beta_i) \|a_i^{t-1}-u_i\|_2^2\right)\\
		&\qquad\qquad+ \frac{1}{q_i}\|a_i^{t-1}-\hat{a}_i\|_2^2\\
		& =\frac{\beta_i}{q_i} \left(\|a_i^{t-1}-\hat{a}_i\|_2^2-\|u_i-\hat{a}_i\|_2^2+(1-\beta_i)\|m_i\|_2^2\right).  
		\end{split}
		\end{equation}
		Taking expectation on both sides with respect to the random sampling of $i$ at time step t, we established Equation~\eqref{equ.xu.thm1-step1}.
		
		\textbf{Step 2.} We now look at the evolution of $B_t$. In particular, we will prove that
		\begin{equation}\label{equ.xu.thm1-step2}\mathbb{E}(B_{t-1}-B_{t})=2\langle w^{t-1}-\hat{w}, \eta ( \nabla \phi(w^{t-1})+\tilde{\lambda} v^{t-1}) \rangle- \frac{\eta^2}{(n+1)^2}\sum_{i=1}^{n+1} \frac{1}{q_i}\|m_i\|_2^2.
		\end{equation}
		To this end, notice that
		\begin{equation}\label{evolution.B_t}
		\begin{split}
		& \frac{1}{2}(B_{t-1}-B_{t})\\
		&=\tilde{g}^*(v^{t-1})-\langle \nabla \tilde{g}^*(\hat{v}),v^{t-1}-\hat{v} \rangle-\tilde{g}^*(\hat{v})- (\tilde{g}^*(v^t)-\langle \nabla \tilde{g}^*(\hat{v}),v^t-\hat{v} \rangle-\tilde{g}^*(\hat{v}))\\
		&=\tilde{g}^*(v^{t-1})-\tilde{g}^*(v^t)-\langle \nabla \tilde{g}^*(\hat{v}),v^{t-1}-v^t \rangle\\
		&\geq \langle \nabla \tilde{g}^*(v^{t-1}), v^{t-1}-v^{t} \rangle-\frac{1}{2}\|v^{t-1}-v^{t}\|_2^{2}-\langle \nabla \tilde{g}^*(\hat{v}),v^{t-1}-v^t \rangle\\
		&= \langle\nabla \tilde{g}^*(v^{t-1})- \nabla \tilde{g}^*(\hat{v}),v^{t-1}-v^t \rangle-\frac{1}{2}\|v^{t-1}-v^t\|_2^2\\
		&=\langle w^{t-1}-\hat{w}, v^{t-1}-v^t \rangle-\frac{1}{2}\|v^{t-1}-v^t\|_2^2,
		\end{split}
		\end{equation}
		where the inequality  holds because $$ \tilde{g}^*(v^t)-\tilde{g}^*(v^{t-1})\leq \langle \nabla \tilde{g}^* (v^{t-1}),v^t-v^{t-1} \rangle+\frac{1}{2}\|v^t-v^{t-1}\|_2^2,$$  which results from $\tilde{g}^*(v)$ being $1-$smooth; and
		the last equality holds from the fact that $ \nabla\tilde{g}^* (v^{t-1})=w^{t-1} $ using Lemma \ref{Lemma.conjugate}.

		Take  expectation on both sides of Equation~\eqref{evolution.B_t}, we established Equation~\eqref{equ.xu.thm1-step2},
		where we use result in Equation \eqref{unbias_estimator}.
		
		\textbf{Step 3}. We define a new potential  $ C_t=c_a A_t+c_b B_t$, and prove in this step that \begin{equation}\label{mid_result}
		\begin{split}
		& \mathbb{E}(C_{t-1}-C_t)\geq  c_a\eta\tilde{\lambda} A_{t-1}-c_a\eta \tilde{\lambda} \sum_{i=1}^{n+1} \frac{1}{q_i} \|u_i-\hat{a}_i\|_2^2+c_b\eta \tilde{\lambda}  B_{t-1}\\ 
		&\qquad\qquad +2c_b\eta(F(w^{t-1})-F(\hat{w}))
		+2c_b\eta ( \phi(\hat{w})-\phi(w^{t-1})-\langle \hat{w}-w^{t-1},\nabla \phi(w^{t-1}) ).
		\end{split}
		\end{equation}
		From the definition of $C_t$, and using Equation~\eqref{equ.xu.thm1-step1} and~\eqref{equ.xu.thm1-step2}, we have
		\begin{equation}\label{C_inequality}
		\begin{split}
		\mathbb{E}(C_{t-1}-C_t)\geq & c_a\eta\tilde{\lambda} A_{t-1}-c_a\eta \tilde{\lambda} \sum_{i=1}^{n+1} \frac{1}{q_i} \|u_i-\hat{a}_i\|_2^2+2c_b\eta \langle w^{t-1}-\hat{w},\nabla \phi(w^{t-1})+\tilde{\lambda}v^{t-1} \rangle \\
		&\quad+\sum_{i=1}^{n+1} \frac{1}{q_i}\|m_i\|_2^2 (c_a\eta \tilde{\lambda} (1-\beta_i)-\frac{c_b \eta^2}{(n+1)^2}).
		\end{split}
		\end{equation}
		We next show that 
		$$ \langle w^{t-1}-\hat{w},\nabla \phi(w^{t-1})+\tilde{\lambda}v^{t-1}\rangle -\frac{\tilde{\lambda}}{2}B_{t-1}-(F(w^{t-1})-F(\hat{w}))=\phi(\hat{w})-\phi(w^{t-1})-\langle \hat{w}-w^{t-1},\nabla \phi(w^{t-1})\rangle.  $$
		This holds by directly verifying as follows:
		\begin{equation}\label{eq.1.theorem}
		\begin{split}
		&\langle w^{t-1}-\hat{w},\nabla \phi(w^{t-1})+\tilde{\lambda}v^{t-1}\rangle -\frac{\tilde{\lambda}}{2}B_{t-1}-(F(w^{t-1})-F(\hat{w}))\\
		=& \langle w^{t-1}-\hat{w},\nabla \phi(w^{t-1})+\tilde{\lambda}v^{t-1}\rangle-\tilde{\lambda} (\tilde{g}^*(v^{t-1})-\langle \nabla \tilde{g}^*(\hat{v}),v ^{t-1}-\hat{v} \rangle-\tilde{g}^*(\hat{v}))-F(w^{t-1})+F(\hat{w})\\
		=& \langle w^{t-1},\tilde{\lambda} v^{t-1} \rangle-\tilde{\lambda} \tilde{g}^*(v^{t-1}) -\langle \hat{w},\tilde{\lambda} \hat{v} \rangle+\tilde{\lambda} \tilde{g}^*(\hat{v})+\langle \hat{w},-\nabla \phi(w^{t-1}) \rangle-F(w^{t-1})+F(\hat{w})\\
		&\quad+ \langle w^{t-1},\nabla \phi(w^{t-1}) \rangle\\
		=&\tilde{\lambda} \tilde{g}(w^{t-1})-\tilde{\lambda} \tilde{g}(\hat{w})-F(w^{t-1})+F(\hat{w})+\langle w^{t-1}-\hat{w},\nabla \phi(w^{t-1})\rangle \\
		=&\tilde{\lambda} \tilde{g}(w^{t-1})-\tilde{\lambda} \tilde{g}(\hat{w})-\tilde{\lambda} \tilde{g}(w^{t-1})-\phi(w^{t-1})+\tilde{\lambda} \tilde{g}(\hat{w})+\phi(\hat{w})+\langle w^{t-1}-\hat{w},\nabla \phi(w^{t-1}) \rangle\\
		=& \phi(\hat{w})-\phi(w^{t-1})-\langle \hat{w}-w^{t-1},\nabla \phi(w^{t-1})\rangle,
		\end{split}
		\end{equation}
		where the second equality uses the fact that $ \hat{w}=\nabla \tilde{g}^*(\hat{v})$, and the third equality holds using the definition of $\tilde{g}^*(v)$.
		Thus, substituting the equation into~\eqref{C_inequality}, we get
		\begin{equation}
		\begin{split}
		& \mathbb{E}(C_{t-1}-C_t)\geq  c_a\eta\tilde{\lambda} A_{t-1}-c_a\eta \tilde{\lambda} \sum_{i=1}^{n+1} \frac{1}{q_i} \|u_i-\hat{a}_i\|_2^2+c_b\eta \tilde{\lambda}  B_{t-1}+2c_b\eta(F(w^{t-1})-F(\hat{w}))\\
		&+\sum_{i=1}^{n+1} \frac{1}{q_i}\|m_i\|_2^2 (c_a\eta \tilde{\lambda} (1-\beta_i)-\frac{c_b \eta^2}{(n+1)^2})+2c_b\eta [ \phi(\hat{w})-\phi(w^{t-1})-\langle \hat{w}-w^{t-1},\nabla \phi(w^{t-1}) ].\\
		\end{split}
		\end{equation}
		We can choose $\eta \leq \frac{q_i}{2\tilde{\lambda} }$ and $\frac{c_b}{c_a}=\frac{\tilde{\lambda} (n+1)^2}{2\eta} $ so that $\beta_i\leq 1/2$, and the term $\sum_{i=1}^{n+1} \frac{1}{q_i}\|m_i\|^2 (c_a\eta \tilde{\lambda} (1-\beta_i)-\frac{c_b \eta^2}{(n+1)^2}) $ is non-negative. Since $q_i\geq \frac{1}{2(n+1)}$ for every i, we can choose $\eta\leq \frac{1}{4\tilde{\lambda} (n+1)}.$
		Using these condition, we established~\eqref{mid_result}.

		\textbf{Step 4}  We now bound $ \sum_{i=1}^{n+1}  \frac{1}{q_i} \|u_i-\hat{a}_i\|_2^2 $.  
		Notice that for $i=1,...,n$, since $\phi_i $ is convex, we can apply Lemma \ref{Lemma.smooth}, and the fact that $\xi+\nabla f(\hat{w})=0$, for $\xi\in \lambda \partial  g(\hat{w})$.
		\begin{equation}
		\begin{split}
		\sum_{i=1}^{n}\frac{1}{q_i} \|u_i-\hat{a}_i\|_2^2 & =\sum_{i=1}^{n} \frac{1}{q_i}\| \nabla \phi_i(w^{t-1})-\nabla\phi_i(\hat{w}) \|_2^2\\
		&\leq\left(2\max \frac{\tilde{L}_i}{q_i}\right) \sum_{i=1}^{n} \left(\phi_i(w^{t-1})-\phi_i(\hat{w})-\langle \nabla \phi_i(\hat{w}),w^{t-1}-\hat{w}  \rangle \right)\\
		&=\frac{n+1}{n} \big(2\max \frac{\tilde{L}_i}{q_i}\big) \sum_{i=1}^{n} \big(f_i(w^{t-1})-f_i(\hat{w})-\langle \nabla f_i(\hat{w}),w^{t-1}-\hat{w}  \rangle \big) \\
		&\leq \frac{n+1}{n} \big(2\max \frac{\tilde{L}_i}{q_i}\big) n\big( f(w^{t-1})-f(\hat{w})+\lambda g(w^{t-1})-\lambda g(\hat{w}) \big) \\
		&\leq \big(2\max \frac{\tilde{L}_i}{q_i}\big)(n+1) (F(w^{t-1})-F(\hat{w})).
		\end{split}
		\end{equation}
		As for $i=n+1$, we have
		\begin{equation}\label{bound_gradient_phi}
		\begin{split}
		\frac{1}{q_{n+1}} \|\nabla \phi_{n+1}(w^{t-1})-\nabla \phi_{n+1}(\hat{w})\|_2^2&\leq \frac{\tilde{\lambda}^2 (n+1)^2}{q_{n+1}} \|w^{t-1}-\hat{w}\|_2^2\\
		&=2(n+1)\frac{\tilde{L}_{n+1}}{q_{n+1}}\frac{\tilde{\lambda}}{2} \|w^{t-1}-\hat{w}\|_2^2.\\
		\end{split}
		\end{equation}
		
		\textbf{	Step 5} We now analyze the progress of the potential by induction.		
		We  need to relate $\frac{\tilde{\lambda}}{2}\|w^{t-1}-\hat{w}\|_2^2$ to $F(w^{t-1})-F(\hat{w}).$ In high level, we divide the time steps $t=1,2,...$ into several epochs, i.e., $ ([ T_0,T_1), [T_1,T_2),...)$. At the end of  each epoch $j$,  we prove that $C_t$ decreases with linear rate until the optimality gap $F(w^t)-F(\hat{w})$ decrease to some tolerance $\xi_j$. We then prove that $(\xi_1,\xi_2,\xi_3,...)$ is a decreasing sequence and finish the proof.
		
		Assuming that time step $t-1$ is in the epoch $j$, we use Equation \eqref{RSC_modified} in Lemma \ref{lemma.RSC_cone} and the fact that $ \tilde{\lambda}\leq \tilde{\kappa}$  to obtain
		$$\frac{1}{q_{n+1}} \|\nabla \phi_{n+1}(w^{t-1})-\nabla \phi_{n+1}(\hat{w})\|_2^2\leq 2(n+1)\frac{\tilde{L}_{n+1}}{q_{n+1}} (F(w^{t-1})-F(\hat{w})+\epsilon_j^2 (\Delta^*,\mathcal{A},\mathcal{B})). $$
		Combining the above two results on $\phi_i(\cdot)$ together, we have
		\begin{equation}
		\begin{split}
		\sum_{i=1}^{n+1}\frac{1}{q_i} \|u_i-\hat{a}_i\|_2^2\leq & 4(n+1) (\max_{ i\in \{1,..,n+1 \}} \frac{\tilde{L_i}}{q_i} ) (F(w^{t-1})-F(\hat{w})+\frac{1}{2}\epsilon_j^2 (\Delta^*,\mathcal{A},\mathcal{B}))\\
		\leq & 8(n+1)^2\tilde{L} (F(w^{t-1})-F(\hat{w}))+4(n+1)^2\tilde{L}\epsilon_j^2 (\Delta^*,\mathcal{A},\mathcal{B}),
		\end{split}
		\end{equation}
		where we use the fact  $\frac{\tilde{L}_i}{q_i}=2(n+1)\tilde{L}\frac{\tilde{L}_i}{\tilde{L}_i+\tilde{L}}\leq 2(n+1) \tilde{L}$.

		Replace the corresponding term in equation \eqref{mid_result}, we have
		\begin{equation}
		\begin{split}
		& \mathbb{E}(C_{t-1}-C_t)\geq  c_a\eta\tilde{\lambda} A_{t-1}-8c_a\eta \tilde{\lambda}(n+1)^2\tilde{L}  (F(w^{t-1})-F(\hat{w}))-4c_a\eta \tilde{\lambda}  (n+1)^2\bar{L}\epsilon_j^2 (\Delta^*,\mathcal{A},\mathcal{B})\\
		&\quad  +c_b\eta \tilde{\lambda}  B_{t-1}+2c_b\eta(F(w^{t-1})-F(\hat{w}))+2c_b\eta ( \phi(\hat{w})-\phi(w^{t-1})-\langle \hat{w}-w^{t-1},\nabla \phi(w^{t-1}) )\\
		&\geq \eta \tilde{\lambda} C_{t-1}+\eta(2c_b-8c_a\tilde{\lambda} (n+1)^2\tilde{L})(F(w^{t-1})-F(\hat{w})) -4c_a\eta \tilde{\lambda}(n+1)^2\tilde{L} \epsilon_j^2 (\Delta^*,\mathcal{A},\mathcal{B})\\
		&\quad +(c_b\eta (\tilde{\kappa}-\tilde{\lambda}) \|w^t-\hat{w}\|_2^2 -2c_b\eta \epsilon_j^2 (\Delta^*,\mathcal{A},\mathcal{B})),
		\end{split}
		\end{equation}
		where the second inequality is due to Lemma \ref{lemma.RSC_cone}.
		We choose $2c_b=16c_a \tilde{\lambda} (n+1)^2\tilde{L} $, and use the fact that $\frac{c_b}{c_a}=\frac{\tilde{\lambda} (n+1)^2}{2\eta}$, we have $\eta=\frac{1}{16\tilde{L}}$ and 
		$$ \mathbb{E}(C_{t-1}-C_t)\geq \eta \tilde{\lambda} C_{t-1}+ \eta c_b (F(w^{t-1})-F(\hat{w}))-3c_b\eta \epsilon_j^2 (\Delta^*,\mathcal{A},\mathcal{B}),$$
		where we use the assumption that $\tilde{\kappa}\geq \tilde{\lambda}$ and the assumption $n>4$.
		
		Remind that in   epoch $j$, we have
		\begin{equation*}\begin{split}
		\epsilon_j^2(\Delta^*,\mathcal{A},\mathcal{B})&=2\tau (\delta_{stat}+\delta_{j-1})^2, \delta_{j-1}=2\min \{\frac{\xi_{j-1}}{\lambda},\rho\},\\
		\mbox{and}\qquad \delta_{stat}&= 8\Psi(\mathcal{B})\|\Delta^*\|+8g(w^*_{\mathcal{A}^\perp}).
		\end{split}\end{equation*}
		The epoch is determined by the comparison between $F(w^{t})-F(\hat{w}) $ and $3\epsilon_j^2 (\Delta^*,\mathcal{A},\mathcal{B})  $
		
		%

		In the first epoch, we choose $\delta_0=2\rho$. 
		Thus $\epsilon_1^2 (\Delta^*,\mathcal{A},\mathcal{B})=2\tau (\delta_{stat}+2\rho)^2. $
		We choose $T_1$ such that 
		
		$$ F(w^{T_1-1})-F(\hat{w})\geq 3\epsilon_1^2 (\Delta^*,\mathcal{A},\mathcal{B}) \quad \text{and} \quad F(w^{T_1})-F(\hat{w})\leq 3\epsilon_1^2 (\Delta^*,\mathcal{A},\mathcal{B}).$$ 
		
		If no such $T_1$ exists, that means $ F(w^{t})-F(\hat{w})\geq 3\epsilon_1^2 (\Delta^*,\mathcal{A},\mathcal{B})$ always holds and we have $ \mathbb{E}(C_t)\leq (1-\eta \tilde{\lambda}) C_{t-1}$ for all $t$ and get the geometric convergence rate. It is a contradiction with $ F(w^{t})-F(\hat{w})\geq 3\epsilon_1^2 (\Delta^*,\mathcal{A},\mathcal{B})$.
		
		Now we know $F(w^{T_1})-F(\hat{w})\leq 3\epsilon_1^2 (\Delta^*,\mathcal{A},\mathcal{B})$,  and hence we choose $\xi_1= 6\tau (\delta_{stat}+\delta_0)^2$.
		
		In the second epoch we use the same argument:
		$$ \epsilon_2^2 (\Delta^*,\mathcal{A},\mathcal{B})=2\tau_\sigma (\delta_{stat}+\delta_1)^2 \quad \text{where} \quad  \delta_1=2 \min \{\frac{\xi_1}{\lambda},\rho\}. $$
		We choose $T_2$ such that 
		$$ F(w^{T_2-1})-F(\hat{w})\geq 3\epsilon_2^2 (\Delta^*,\mathcal{A},\mathcal{B}) \quad \text{and} \quad F(w^{T_2})-F(\hat{w})\leq 3\epsilon_2^2 (\Delta^*,\mathcal{A},\mathcal{B}).$$ 
		Then we choose $\xi_2= 6\tau (\delta_{stat}+\delta_1)^2$.
		
		Similarly in epoch $j$, we choose $T_j$ such that
		$$ F(w^{T_j-1})-F(\hat{w})\geq 3\epsilon_j^2 (\Delta^*,\mathcal{A},\mathcal{B}) \quad \text{and} \quad F(w^{T_j})-F(\hat{w})\leq 3\epsilon_j^2 (\Delta^*,\mathcal{A},\mathcal{B}),$$ 
		and $ \xi_j= 6\tau (\delta_{stat}+\delta_{j-1})^2 $.
		
		In this way, we arrive at recursive equalities of the tolerance $ \{\xi_j\}_{j=1}^\infty$ where 
		$$\xi_j= 6\tau_\sigma (\delta_{stat}+\delta_{j-1})^2 \quad \text{and} \quad \delta_{j-1}=2 \min \{\frac{\xi_{j-1}}{\lambda},\rho\}.$$
		We claim that following holds, until $\delta_j=\delta_{stat}$. 
		\begin{equation}\label{equ.xu.thm1step5}
		\begin{split}
		(I)\quad &\xi_{k+1}\leq \xi_k/ (4^{2^{k-1}}) \\
		\mbox{and}\quad  (II)\quad &\quad\frac{\xi_{k+1}}{\lambda}\leq \frac{\rho}{4^{2^k}} \quad for \quad k=1,2,3,... 
		\end{split}
		\end{equation}
		
		The proof of Equation~\eqref{equ.xu.thm1step5} is same with Equation (60) in \cite{agarwal2010fast}, which we present here for  completeness.
		
		We assume $\delta_0\geq \delta_{stat}$ (otherwise the statement is true trivially), so $\xi_1\leq 96 \tau \rho^2$. We make the assumption that $ \lambda\geq 384\tau \rho $, so $ \frac{\xi_1}{\lambda}\leq \frac{\rho}{4}$ and $ \xi_1\leq \xi_0$. 
		
		In the second epoch we have $$\xi_2 \overset{(1)}{\leq} 12\tau (\delta^2_{stat}+\delta^2_1)\leq 24\tau \delta_1^{2} \leq \frac{96 \tau \xi_1^2}{\lambda^2}\overset{(2)}{\leq} \frac{96\tau \xi_1}{4\lambda}\overset{(3)}{\leq} \frac{\xi_1}{4},$$
		where (1) holds from the fact that $(a+b)^2\leq 2a^2+2b^2$, (2) holds using $ \frac{\xi_1}{\lambda}\leq \frac{\rho}{4}$, (3) uses the assumption on $\lambda$. Thus,
		$$ \frac{\xi_2}{\lambda}\leq \frac{\xi_1}{4\lambda} \leq \frac{\rho}{16}.$$
		In $i+1$th step, with similar argument, and by induction assumption we have
		$$\xi_{j+1}\leq \frac{96\tau\xi_j^2}{\lambda^2}\leq \frac{96\tau \xi_j}{4^{2^{j}}\lambda}\leq \frac{\xi_j}{4^{2^{k-1}}} $$
		and 
		$$ \frac{\xi_{j+1}}{\lambda}\leq \frac{\xi_j}{4^{2^{j-1}} \lambda}\leq \frac{\rho}{4^{2^j}}.$$
		Thus we know $\xi_j$ is a decreasing sequence, and $\mathbb{E} (C_t)\leq (1-\eta \tilde{\lambda}) C_{t-1}$ holds until $F(w^t)-F(\hat{w})\leq 6\tau (2\delta_{stat})^2$, where $\eta$ is set to satisfy $\eta\leq \min (\frac{1}{16\tilde{L}}, \frac{1}{4\tilde{\lambda}(n+1)}).$  
		
		Observe that $ (n+1)\tilde{L}=\frac{n+1}{n}(\sum_{i=1}^{n}L_i)+\lambda (n+1)=(n+1) (\lambda +\frac{1}{n} \sum_{i=1}^{n} L_i)$, we establish the theorem.
	\end{proof}

	\subsection{Strongly convex $F(w)$}
	
	\begin{proof}[Proof of Proposition 1]
		
		If $g(w)$ is 1 strongly convex, we can directly apply the algorithm on Equation \eqref{obj}. In this case,  RSC is not needed and the proof is significantly simpler. As the main road map of the proof is similar, we just mention the difference in the following.
		We define the potential,
		\begin{equation*}\begin{split}
		A_t &=\sum_{j=1}^{n}\frac{1}{q_i}\|a_j^t-\hat{a}_j\|_2^2, \\
		B_t & =2(g^*(v^t)-\langle \nabla g^*(\hat{v}),v^t-\hat{v} \rangle-g^*(\hat{v})).
		\end{split}\end{equation*}
		Notice in this setup, the potential $B_t$ is defined on $g$ rather than $\tilde{g}$.

		The evolution of $A_t$ follows from the same analysis of Step 1 in the proof of Theorem~\ref{main_theorem} , except that we replace $\tilde{\lambda}$ by $\lambda$ and we only have $n$ terms rather than $n+1$ terms. This gives
		$$ \mathbb{E}[A_{t-1}-A_t]=\eta \lambda ( A_{t-1}+\sum_{i=1}^{n}\frac{1}{q_i} (-\|u_i-\hat{a}_i\|_2^2+(1-\beta_i)\|m_i\|_2^2)  ). $$
		Note that $g^*(v)$ is 1-smooth, follow the same analysis of Step 2  in the proof of Theorem~\ref{main_theorem}, specifically the derivation in \eqref{evolution.B_t}, we obtain
		$$ \frac{1}{2}(B_{t-1}-B_{t})\geq \langle w^{t-1}-\hat{w}, v^{t-1}-v^t \rangle-\frac{1}{2}\|v^{t-1}-v^t\|_2^2. $$
		Combining the above two equations, we have
		\begin{equation}\label{evolution.B_t_easy}
		\begin{split}
		\mathbb{E}(C_{t-1}-C_t)\geq & c_a\eta\lambda A_{t-1}-c_a\eta \lambda \sum_{i=1}^{n} \frac{1}{q_i} \|u_i-\hat{a}_i\|_2^2+2c_b\eta \langle w^{t-1}-\hat{w},\nabla f(w^{t-1})+\lambda v^{t-1} \rangle \\
		&+\sum_{i=1}^{n} \frac{1}{q_i}\|m_i\|_2^2 (c_a\eta \lambda (1-\beta_i)-\frac{c_b \eta^2}{n^2}).
		\end{split}
		\end{equation}
		We can choose $\eta \leq \frac{q_i}{2\lambda}$ and $\frac{c_b}{c_a}=\frac{\lambda n^2}{2\eta} $ so that $\beta\leq 1/2$, and the term $\sum_{i=1}^{n} \frac{1}{q_i}\|m_i\|_2^2 (c_a\eta \lambda (1-\beta_i)-\frac{c_b \eta^2}{n^2}) $ is non-negative. According to the definition of $q_i$, we know $q_i \geq \frac{1}{2n}$, thus we can choose  $\eta\leq \frac{1}{4n} $ to ensure that $\eta \leq \frac{q_i}{2\lambda}$. 
		
		We can show that 
		\begin{equation}\label{claim}
		\begin{split}
		&\langle w^{t-1}-\hat{w},\nabla f(w^{t-1})+\lambda v^{t-1}\rangle -\frac{\lambda}{2}B_{t-1}-(F(w^{t-1})-F(\hat{w}))\\
		=&f(\hat{w})-f(w^{t-1})-\langle \hat{w}-w^{t-1},\nabla f(w^{t-1})\rangle\geq 0
		\end{split}
		\end{equation}
		This follows from the same derivation as in Equation \eqref{eq.1.theorem} , modulus replacing $\tilde{\lambda}$, $\tilde{g}(w)$ and $\phi_i(w)$ by $\lambda$, $g(w)$ and $f_i(w)$ correspondingly. The last step follows immediately from the convexity of $f(w)$.
		
		The proof now proceed with discussing two cases separately: (1). each $f_i(w)$ is convex. And (2). $f_i(w)$ is not necessarily convex, but the sum $f(w)$ is convex.
		
		\textbf{Case 1:} Since each $f_i(w)$ is convex in this case,
		to bound $\sum_{i=1}^{n} \frac{1}{q_i} \|u_i-\hat{a}_i\|_2^2 $ is easier, via the following. 
		\begin{equation}
		\begin{split}
		\sum_{i=1}^{n} \frac{1}{q_i} \|u_i-\hat{a}_i\|_2^2 & = \sum_{i=1}^{n} \frac{1}{q_i} \|\nabla f_i(w^{t-1})-\nabla f_i(\hat{w}) \|_2^2\\
		& \leq (2 \max_i \frac{L_i}{q_i})\sum_{i=1}^{n} (f_i(w^{t-1})-f(\hat{w})-\langle \nabla f_i(\hat{w}),w^{t-1}-\hat{w}\rangle  )\\
		&\leq (2 \max_i \frac{L_i}{q_i}) n (F (w^{t-1})-F(\hat{w})),
		\end{split}
		\end{equation}
		where the first inequality follows from Lemma~\ref{Lemma.smooth}, and the second one follows from convexity of $g(w)$ and optimality condition of $\hat{w}$.
		
		The definition of $q_i$ in the Algorithm 1  implies that for every $i$,
		$$ \frac{L_i}{q_i}=2n\bar{L}\frac{L_i}{L_i+L}\leq 2n \bar{L}.   $$
		Substitute this into the corresponding terms in \eqref{evolution.B_t_easy}, we have
		\begin{equation}
		\begin{split}
		&\mathbb{E}(C_{t-1}-C_t)\\
		\geq &c_a \eta \lambda A_{t-1}+ c_b \eta \lambda B_{t-1} +2c_b\eta (F(w^{t-1})-F(\hat{w}))-4c_a\eta\lambda  n^2\bar{L} (F(w^{t-1})-F(\hat{w})).
		\end{split}
		\end{equation}
		Remind $c_b=\frac{c_a \lambda n^2}{2\eta}$, using the choice $\eta\leq \frac{1}{4 \bar{L}}$, we have
		$$ E(C_t)\leq (1-\eta\lambda) C_{t-1}, $$
		where $\eta=\min \{\frac{1}{4\bar{L}},\frac{1}{4\lambda n}\} $.
		
		\textbf{Case 2:}  Using  strong convexity of $F(\cdot)$, we have 
		\begin{equation}
		\begin{split}
		& c_a\eta \lambda\sum_{i}^{n} \frac{1}{q_i} \|u_i-\hat{a}_i\|_2^2=c_a\eta\lambda \sum_{i=1}^{n} \frac{1}{q_i} \|\nabla f_i(w^{t-1})-\nabla f_i(\hat{w})\|_2^2 \\
		\leq  & c_a\eta\lambda \sum_{i=1}^{n} \frac{L_i^2}{q_i} \|w^{t-1}-\hat{w}\|_2^2\leq 2c_a\eta(F(w^{t-1})-F(\hat{w})) \sum_{i=1}^{n} \frac{L^2_i}{q_i},
		\end{split}
		\end{equation} 
		where the first inequality uses the smoothness of $f_i(\cdot)$, and the second one holds from the strong convexity of $F(\cdot)$ and optimal condition of $\hat{w}$.
		
		Using Equation \eqref{claim}, we have 
		$$\langle w^{t-1}-\hat{w},\nabla f(w^{t-1})+\lambda v^{t-1}\rangle \geq \frac{\lambda}{2}B_{t-1}+(F(w^{t-1})-F(\hat{w})). $$
		Then replace the corresponding terms in \eqref{evolution.B_t_easy} we have
		$$ \mathbb{E} (C_{t-1}-C_{t})\geq \eta \lambda C_{t-1} + 2\eta(c_b- c_a \sum_{i=1}^{n} \frac{L_i^2}{q_i})(F(w^{t-1})-F(\hat{w})). $$
		Since we know $\frac{L_i}{q_i}\leq 2n\bar{L}$, we have
		$$\mathbb{E} (C_{t-1}-C_{t})\geq \eta \lambda C_{t-1}+2\eta (c_b-2n^2\bar{L}^2 c_a) (F(w^{t-1})-F(\hat{w})).  $$
		
		The last term is no-negative if $\frac{c_b}{c_a}\geq 2n^2\bar{L}^2$. Since we choose $\frac{c_b}{c_a}=\frac{\lambda n^2}{2\eta}$, this condition is satisfied as we choose $\eta\leq \frac{\lambda}{4\bar{L}^2}$.
	\end{proof}

	\subsection{Proof on non-convex $F(w)$}
	The proof of the non-convex case follows a similar line as that of the convex case. To avoid redundancy,  we focus on pointing out the difference in the proof. We start with some technical lemmas. The following lemma is adapted from Lemma 6 of \cite{loh2013regularized}, which  we present for completeness.
	\begin{lemma}\label{lemma.non_convex_norm}
		For any vector $ w\in R^p$, let $A$ denote the index set of its $s$ largest elements in magnitude, under assumption of $d_{\lambda,\mu}$ in Section \ref{section:Assumption_non_convex} in the main body of paper, we have 
		$$ d_{\lambda,\mu}(w_A)-d_{\lambda,\mu} (w_{A^{c}})\leq \lambda L_d (\|w_A\|_1-\|w_{A^c}\|_1)  .$$ 
		Moreover, for an arbitrary vector $w\in \mathbb{R}^p$, we have
		$$ d_{\lambda,\mu} (w^*)-d_{\lambda,\mu} (w)\leq \lambda L_d (\|\nu_A\|_1-\|\nu_{A^c}\|_1), $$
		where $\nu=w-w^*$ and $w^*$ is $s$ sparse.
	\end{lemma}
	The next  lemma is non-convex counterparts of Lemma \ref{lemma.cone} and Lemma \ref{lemma.cone_optimization}. 
	\begin{lemma}\label{lemma.non_convex_cone_optimizatin}
		Suppose $d_{\lambda,\mu}(\cdot)$ satisfies the assumptions in Section \ref{section:Assumption_non_convex}  in the main body of paper, $w^*$ is feasible,   $\lambda L_d\geq8\rho\theta \frac{\log p }{n}$,  $\lambda\geq \frac{4}{L_d} \|\nabla f (w^*)\|_\infty$, and there exists $\xi, T$ such that 
		$$F(w^t)-F(\hat{w})\leq \xi, \forall t> T.$$
		Then for any $t> T$, we have	
		$$ \|w^t-\hat{w}\|_1\leq 4\sqrt{s} \|w^t-\hat{w}\|_2+8\sqrt{s} \|w^*-\hat{w}\|_2+2\min \big(\frac{\xi}{\lambda L_d}, \rho\big).$$
	\end{lemma}
	\begin{proof}
		For an arbitrary feasible $w$, define $\Delta=w-w^* $.
		Suppose we have $F(w)-F(\hat{w})\leq \xi$, since we know $F(\hat{w})\leq F(w^*) $ so we have $ F(w)\leq F(w^*)+\xi$, which implies 
		$$ f(w^*+\Delta)+d_{\lambda,\mu} (w^*+\Delta)\leq f(w^*)+d_{\lambda,\mu}(w^*) +\xi .$$
		Subtract $\langle \nabla f(w^*),\Delta \rangle$ and use the RSC condition (Recall we have $\tau=\theta\frac{\log p}{n}$ in the assumption of Theorem 2) we have
		\begin{equation}
		\begin{split}
		&\frac{\kappa}{2} \|\Delta\|_2^2-\theta\frac{\log p}{n} \|\Delta\|_1^2+d_{\lambda,\mu} (w^*+\Delta)-d_{\lambda,\mu}(w^*)\\
		\leq &\xi-\langle \nabla f(w^*),\Delta \rangle\\
		\leq &\xi+\|\nabla f(w^*)\|_\infty \|\Delta\|_1,
		\end{split}
		\end{equation}
		where the last inequality holds from Holder's inequality.

		
		Rearranging terms and use the fact that $\|\Delta\|_1\leq 2\rho $ (by feasibility of $w$ and $w^*$) and  we assume that $\lambda L_d\geq 8\rho\theta \frac{\log p }{n}$ , $\lambda\geq \frac{4}{L_d} \|\nabla f (w^*)\|_\infty$,   we get
		$$ \xi+\frac{1}{2} \lambda L_{d} \|\Delta\|_1+d_{\lambda,\mu}(w^*)-d_{\lambda,\mu}(w^*+\Delta)\geq \frac{\kappa}{2}\|\Delta\|_2^2\geq 0. $$
		By Lemma \ref{lemma.non_convex_norm}, we have
		$$ d_{\lambda,\mu} (w^*)-d_{\lambda,\mu} (w)\leq \lambda L_d (\|\Delta_A\|_1-\|\Delta_{A^c}\|_1) ,$$
		where $A$ is the set of indices of the top $s$ components of $\Delta$ in magnitude, which thus leads to
		$$ \frac{3\lambda L_d}{2} \|\Delta_A\|_1-\frac{\lambda L_d}{2} \|\Delta_{A^c}\|_1+\xi\geq 0. $$
		Consequently 
		$$ \|\Delta\|_1\leq \|\Delta_A\|_1+\|\Delta_{A^c}\|_1\leq 4\|\Delta_A\|_1+\frac{2\xi}{\lambda L_d} \leq 4\sqrt{s} \|\Delta\|_2 +\frac{2\xi}{\lambda L_d}.$$
		
		Combining with the fact $  \|\Delta\|_1\leq 2\rho$ , we obtain
		$$ \|\Delta\|_1\leq 4\sqrt{s} \|\Delta\|_2 +2\min \{\frac{\xi}{\lambda L_d}, \rho\}.$$
		So we have $\|w^t-w^*\|_1\leq 4\sqrt{s} \|w^t-w^*\|_2+2\min \{\frac{\xi}{\lambda L_d}, \rho\}.$
		
		Notice $F(w^*)-F(\hat{w})\geq 0$, so following same steps and set $\xi=0$ we have $\|\hat{w}-w^*\|_1\leq 4\sqrt{s} \|\hat{w}-w^*\|_2.$
		
		Combining the two together, we get 
		$$\|w^t-\hat{w}\|_1\leq \|w^t-w^*\|_1+\|w^*-\hat{w}\|_1\leq  4\sqrt{s} \|w^t-\hat{w}\|_2+8\sqrt{s} \|w^*-\hat{w}\|_2+2\min (\frac{\xi}{\lambda L_d}, \rho). $$	
	\end{proof}
	Now we provide a counterpart of Lemma \ref{lemma.RSC_cone} in the non-convex case.
	\begin{lemma}\label{lemma.non_convex_RSC_cone}
		Under the same assumptions as those of Lemma \ref{lemma.non_convex_cone_optimizatin}, we  have
		\begin{equation*}
		\begin{split}
		& F(w^t)-F(\hat{w})\geq \frac{\tilde{\kappa}}{2} \|w^t-\hat{w}\|_2^2-\epsilon^2 (\Delta^*,s );\\
		\mbox{and}\quad & \phi(\hat{w})-\phi(w^t)-\langle \nabla \phi(w^t), \hat{w}-w^t \rangle\geq   [\frac{\tilde{\kappa}-\tilde{\lambda}}{2}\|w^t-\hat{w}\|_2^2-\epsilon^2 (\Delta^*,s )], 
		\end{split}
		\end{equation*}
		where $\tilde{\kappa}=\kappa-\mu-64s\tau$, $\Delta^*=\hat{w}-w^*$, and $\epsilon^2 (\Delta^*,s )=2\tau (8\sqrt{s}\|\hat{w}-w^*\|_2+2\min(\frac{\xi}{\lambda L_d},\rho))^2$.
	\end{lemma}
	\begin{proof} Notice that
		\begin{equation}\label{eq:lemma_non_convex}
		\begin{split}
		& F(w^t)-F(\hat{w})\\
		= & f(w^t)-f(\hat{w})-\frac{\mu}{2}\|w^t\|_2^2+\frac{\mu}{2}\|\hat{w}\|_2^2+ \lambda d_\lambda(w^t)-\lambda d_\lambda(\hat{w})\\
		\geq & \langle \nabla f(\hat{w}),w^t-\hat{w} \rangle+\frac{\kappa}{2}\|w^t-\hat{w}\|_2^2- \langle \mu\hat{w}, w^t-\hat{w} \rangle-\frac{\mu}{2} \|w^t-\hat{w}\|_2^2\\
		&+ \lambda d_\lambda(w^t)-\lambda d_\lambda(\hat{w})-\tau \|w^t-\hat{w}\|_1^2\\
		\geq &  \langle \nabla f(\hat{w}),w^t-\hat{w} \rangle+\frac{\kappa}{2}\|w^t-\hat{w}\|_2^2- \langle \mu\hat{w}, w^t-\hat{w} \rangle-\frac{\mu}{2} \|w^t-\hat{w}\|_2^2\\
		&+ \lambda \langle \partial d_\lambda(\hat{w}),w^t-\hat{w}\rangle-\tau \|w^t-\hat{w}\|_1^2 \\
		= & \frac{\kappa-\mu}{2}\|w^t-\hat{w}\|_2^2-\tau\|w^t-\hat{w}\|_1^2,
		\end{split}
		\end{equation}
		where the first inequality uses RSC condition, the second inequality uses the convexity of $d_\lambda(w)$, and the last equality holds from the optimality condition of $\hat{w}$.

		Using Lemma \ref{lemma.non_convex_cone_optimizatin}, and the inequality $(a+b)^2\leq 2a^2+2b^2$, we have
		\begin{equation}\label{eq:lemma.non_convex_RSC_cone}
		\begin{split}
		\|w^t-\hat{w}\|_1^2\leq  &(4\sqrt{s} \|w^t-\hat{w}\|_2+8\sqrt{s} \|w^*-\hat{w}\|_2+2\min (\frac{\xi}{\lambda L_d}, \rho))^2 \\
		\leq &32 s \|w^t-\hat{w}\|_2^2+2 (8\sqrt{s}\|\hat{w}-w^*\|_2+2\min(\frac{\xi}{\lambda L_d},\rho))^2.
		\end{split}
		\end{equation}
		Substitute this into Equation \eqref{eq:lemma_non_convex}, we obtain
		$$ F(w^t)-F(\hat{w})\geq (\frac{\kappa-\mu}{2}-32s\tau)\|w^t-\hat{w}\|_2^2-2\tau (8\sqrt{s}\|\hat{w}-w^*\|_2+2\min(\frac{\xi}{\lambda L_d},\rho))^2.$$

		Recall 
		$$\phi(w)= f(w)-\frac{\tilde{\lambda}}{2}\|w\|_2^2
		-\frac{\mu}{2}\|w\|_2^2.$$	 
		So we have
		\begin{equation}
		\begin{split}
		&\phi(\hat{w})-\phi(w^{t})-\langle \hat{w}-w^{t},\nabla \phi(w^{t})\rangle\\
		=& f(\hat{w})-f(w^{t})-\langle\nabla f(w^{t}), \hat{w}-w^{t}\rangle-\frac{\tilde{\lambda}+\mu}{2}\|\hat{w}-w^{t}\|_2^2 \\
		\geq& \frac{\kappa}{2} \|\hat{w}-w^t\|_2^2-\tau\|\hat{w}-w^t\|_1^2-\frac{\tilde{\lambda}+\mu}{2}\|\hat{w}-w^t\|_2^2.
		\end{split}
		\end{equation}
		Then use the upper bound of $\|w^t-\hat{w}\|_1^2$ in \eqref{eq:lemma.non_convex_RSC_cone} and rearrange terms we establish the lemma.
	\end{proof}

	We are now ready to prove the main theorem of the non-convex case, i.e., Theorem 2.
	\begin{proof}[Proof of Theorem 2]	
		Remind in the non-convex case, for $i=1,...,n$, the definition of $\phi_i(x)$ is the same as in the convex case.  
		The difference is for $\phi_{n+1}(w) $, where $\phi_{n+1}(w)$ is defined as follows:
		$$ \phi_{n+1}(w)=-\frac{(n+1) (\tilde{\lambda}+\mu)}{2} \|w\|_2^2. $$
		Recall the definition of $\tilde{g}(w)$ is as follows
		$$ \tilde{g}(w)=\frac{1}{2}\|w\|_2^2+\frac{\lambda}{\tilde{\lambda}}d_\lambda (w). $$

		Following similar steps  as in the proof of Theorem~\ref{main_theorem}, we have 
		\begin{equation}
		\begin{split}
		& E(C_{t-1}-C_t)\geq  c_a\eta\tilde{\lambda} A_{t-1}-c_a\eta \tilde{\lambda} \sum_{i=1}^{n+1} \frac{1}{q_i} \|u_i-\hat{a}_i\|_2^2+c_b\eta \tilde{\lambda}  B_{t-1}+2c_b\eta(F(w^{t-1})-F(\hat{w}))\\
		&+\sum_{i=1}^{n+1} \frac{1}{q_i}\|m_i\|_2^2 (c_a\eta \tilde{\lambda} (1-\beta_i)-\frac{c_b \eta^2}{(n+1)^2})+2c_b\eta [ \phi(\hat{w})-\phi(w^{t-1})-\langle \hat{w}-w^{t-1},\nabla \phi(w^{t-1}) ].
		\end{split}
		\end{equation}

		Following the same steps as in Equation~\eqref{bound_gradient_phi}, and replace $\tilde{\lambda}$ by $\tilde{\lambda}+\mu$ we have
		\begin{equation}
		\begin{split}
		\frac{1}{q_{n+1}} \|\nabla \phi_{n+1}(w^{t-1})-\nabla \phi_{n+1}(\hat{w})\|_2^2
		&\leq 2(n+1)\frac{\tilde{L}_{n+1}}{q_{n+1}}\frac{\tilde{\lambda}+\mu}{2} \|w^{t-1}-\hat{w}\|_2^2.\\
		\end{split}
		\end{equation}
		Similar as the convex case we then divide the time step $t=1,2,...$ into several epochs, i.e., $ ([ T_0,T_1), [T_1,T_2),...)$. At the end of  each epoch $j$,  we prove that $C_t$ decreases with a linear rate until the optimality gap $F(w^t)-F(\hat{w})$ decrease to some tolerance $\xi_j$.

		We then apply Lemma \ref{lemma.non_convex_RSC_cone} and using the assumption that $ \tilde{\kappa}\geq \tilde{\lambda}+\mu $ to relate $\|w^{t-1}-\hat{w}\|_2^2 $ to $ F(w^t)-F(\hat{w})$:
		$$\frac{1}{q_{n+1}} \|\nabla \phi_{n+1}(w^{t-1})-\nabla \phi_{n+1}(\hat{w})\|_2^2\leq 2(n+1)\frac{\tilde{L}_{n+1}}{q_{n+1}} (F(w^{t-1})-F(\hat{w})+\epsilon_j^2 (\Delta^*,s)). $$
		Thus
		\begin{equation}
		\begin{split}
		\sum_{i=1}^{n+1}\frac{1}{q_i} \|u_i-\hat{a}_i\|_2^2\leq  8(n+1)^2\tilde{L} (F(w^{t-1})-F(\hat{w}))+4(n+1)^2\tilde{L}\epsilon_j^2 (\Delta^*,s).
		\end{split}
		\end{equation}
		We choose $2c_b=16c_a \tilde{\lambda} (n+1)^2\tilde{L} $, and use the fact that $\frac{c_b}{c_a}=\frac{\tilde{\lambda} (n+1)^2}{2\eta}$, we have $\eta=\frac{1}{16\tilde{L}}$. Combine all pieces together we get
		\begin{equation}
		\begin{split}
		& \mathbb{E}(C_{t-1}-C_t)\geq  c_a\eta\tilde{\lambda} A_{t-1}-8c_a\eta \tilde{\lambda}(n+1)^2\tilde{L}  (F(w^{t-1})-F(\hat{w}))-4c_a\eta \tilde{\lambda}  (n+1)^2\bar{L}\epsilon_j^2 (\Delta^*,s)\\
		&\quad  +c_b\eta \tilde{\lambda}  B_{t-1}+2c_b\eta(F(w^{t-1})-F(\hat{w}))+2c_b\eta ( \phi(\hat{w})-\phi(w^{t-1})-\langle \hat{w}-w^{t-1},\nabla \phi(w^{t-1}) )\\
		&\geq \eta \tilde{\lambda} C_{t-1}+\eta(2c_b-8c_a\tilde{\lambda} (n+1)^2\tilde{L})(F(w^{t-1})-F(\hat{w})) -4c_a\eta \tilde{\lambda}(n+1)^2\tilde{L} \epsilon_j^2 (\Delta^*,s)\\
		&
		\quad+(c_b\eta (\tilde{\kappa}-\tilde{\lambda}) \|w^t-\hat{w}\|_2^2 -2c_b\eta \epsilon_j^2 (\Delta^*,s))\\
		&\geq \eta \tilde{\lambda} C_{t-1}+ \eta c_b (F(w^{t-1})-F(\hat{w}))-3c_b\eta \epsilon_j^2 (\Delta^*,s),
		\end{split}
		\end{equation}
		where the second inequality uses Lemma \ref{lemma.non_convex_RSC_cone} and the assume $\tilde{\kappa}\geq\tilde{\lambda}$.
		
		The rest of the proofs are almost identical to the convex one. In the first epoch, we have
		$$ \epsilon_1^2 (\Delta^*,s)=2\tau (\delta_{stat}+2\rho)^2 ,  \xi_1= 6\tau (\delta_{stat}+2\rho)^2 $$
		Then we choose $T_1$ such that 
		$$ F(w^{T_1-1})-F(\hat{w})\geq 3\epsilon_1^2 (\Delta^*,s) \quad \text{and} \quad F(w^{T_1})-F(\hat{w})\leq 3\epsilon_1^2 (\Delta^*,s).$$ 
		In the second epoch we can use the same argument.
		$$ \epsilon_2^2 (\Delta^*,s)=2\tau_\sigma (\delta_{stat}+\delta_1)^2 \quad \text{where} \quad  \delta_1=2\frac{\xi_1}{\lambda L_d}. $$
		Repeat this strategy on every epoch. The only difference from the proof of the convex case is that we now replace $\lambda$ by $\lambda L_d$. 
		
		We use the same argument to prove $\xi_k$ is a decreasing sequence and conclude the proof.
	\end{proof}

	\subsection{Proof of Corollaries}
	We now prove the corollaries that instantiate our main theorems to different statistical estimators.
	\begin{proof}[Proof of corollary on Lasso, i.e., corollary \ref{cor.lasso}]
		To begin with, we present the following  lemma of the RSC  proved in \cite{raskutti2010restricted} and then use it in the case of Lasso.
		\begin{lemma}
			If each data point $x_i$ is i.i.d randomly sampled from the distribution $N(0,\Sigma) $, then there exist universal constants $c_0$ and $c_1$ such that
			$$ \frac{\|X\Delta\|_2^2}{n}\geq \frac{1}{2}\|\Sigma^{1/2}\Delta\|_2^2-c_1\nu(\Sigma)\frac{\log p}{n} \|\Delta\|_1^2, \quad \mbox{for all} \quad \Delta\in \mathbb{R}^p,$$
			with probability at least $1-\exp(-c_0n)$. Here $X$ is the data matrix where each row is data point $x_i $.
		\end{lemma}
		Since  $w^*$ is supported on a subset $S$ with cardinality s,  we choose $$ \mathcal{B}(S)=\{ w\in \mathbb{R}^p | w_j=0 \text{~for all ~} j\notin S   \}. $$ It is straightforward to choose $\mathcal{A}(S)=\mathcal{B}(S)$ and notice that $w^*\in \mathcal{A}(S)$.
		
		In Lasso, $f(w)=\frac{1}{2n}\|y-Xw\|_2^2$,  and it is easy to verify that
		$$ f(w+\Delta)-f(w)-\langle \nabla f(w),\Delta \rangle\geq \frac{1}{2n} \|X\Delta\|_2^2\geq \frac{1}{4}\|\Sigma^{1/2}\Delta\|_2^2-\frac{c_1}{2}\nu(\Sigma)\frac{\log p}{n} \|\Delta\|_1^2. $$
		Notice that $g(\cdot)$ is $\|\cdot\|_1$ in Lasso, thus $\Psi(\mathcal{B})=\sup_{w\in \mathcal{B}\backslash \{0\}} \frac{\|w\|_1}{\|w\|_2}=\sqrt{s}.$ So we have 
		$$\tilde{\kappa}=\frac{1}{2}\sigma_{\min} (\Sigma)-64c_1 \nu(\Sigma)\frac{ s\log p}{n} .$$
		On the other hand, the tolerance is 
		\begin{equation}
		\begin{split}
		\delta&=24\tau(8\Psi(\mathcal{B})\|\hat{w}-w^*\|_2+8g(w^*_{\mathcal{A}^\perp}))^2\\
		&=c_2\nu (\Sigma)\frac{s\log p}{n}\|\hat{w}-w^*\|_2^2,
		\end{split}
		\end{equation}
		where we use the fact that $w^*\in \mathcal{A}(S)$ which implies  $g(w^*_{\mathcal{A}^{\perp}})=0$. 
		
		The last piece to check is that $\lambda\geq 2g^*(\nabla f(w^*))$. In Lasso we have $g^*(\cdot)=\|\cdot\|_\infty$. Using the fact that $y_i=x_i^Tw^*+\xi_i$, this is equivalent to require $ \lambda\geq \frac{2}{n}\|X^T\xi\|_\infty $. Thus, by our choice of $\lambda$ that $\lambda \geq 6\sigma\sqrt{\frac{\log p}{n}}$, the condition is satisfied invoking the following inequality which holds by applying the tail bound on the Gaussian variable and the union bound to get
		$$\mathbb{P}\left(\frac{2}{n}\|X^T\xi\|_\infty\leq 6\sigma\sqrt{\frac{\log p}{n}}\right) \geq 1-\exp(-3\log p).$$
	\end{proof}
	
	\begin{proof}[Proof of corollary on Group Lasso, i.e., corollary \ref{cor.group_lasso}]
		We use the following fact on the RSC condition of the Group Lasso \cite{negahban2009unified}\cite{negahban2012supplement}. As each data point $x_i$ is i.i.d random sampled from the distribution $N(0,\Sigma) $, then  there exists strictly positive constant  $(\sigma_1,\sigma_2)$ which just depends on $\Sigma$ such that 
		$$ \frac{ \|X\Delta\|_2^2}{2n} \geq \sigma_1(\Sigma) \|\Delta\|_2^2-\sigma_2(\Sigma) (\sqrt{\frac{m}{n}}+\sqrt{\frac{3 \log N_{\mathcal{G}}}{n}} )^2 \|\Delta\|^2_{\mathcal{G},2}, \quad \mbox{for all} \quad \Delta \in \mathbb{R}^{p} $$
		with probability at least $1-c_3\exp (-c_4n)$.
		
		Recall we define the subspace 
		$$ \mathcal{A} (S_{\mathcal{G}})=\{w|w_{G_{i}}=0 ~~ \text{for all } ~~ i\notin S_{\mathcal{G}} \} , $$and $\mathcal{A(S_{\mathcal{G}})}=\mathcal{B (S_{\mathcal{G}}})$, where $S_{\mathcal{G}}$ corresponds to non-zero group of $w^*$. The subspace compatibility can be computed by  $$ \Psi(\mathcal{B})=\sup_{w\in \mathcal{B}\backslash \{0\}} \frac{\|w\|_{\mathcal{G},2}}{\|w\|_2}=\sqrt{s_{\mathcal{G}}}.$$
		The effective RSC parameter is given by		
		$$\tilde{\kappa}=\sigma_1(\Sigma)-64\sigma_2(\Sigma)s_{\mathcal{G}} \left(\sqrt{\frac{m}{n}}+\sqrt{\frac{3 \log N_{\mathcal{G}}}{n}} \right)^2  .$$
		We then check the requirement on $\lambda$ holds.  Since  Group lasso is $\ell_{1,2}$ grouped norm, its dual norm of it is $(\infty,2)$ grouped norm.
		So we need $$\lambda\geq 2 \max_{i=1,...,N_{\mathcal{G}}} \|\frac{1}{n}(X^T\xi)_{G_i}\|_2.$$
		Using Lemma 5 in \cite{negahban2009unified}, we know 
		$$\max_{i=1,...,N_{\mathcal{G}}} \left\|\frac{1}{n}(X^T\xi)_{G_i}\right\|_2\leq 2\sigma \left(\sqrt{\frac{m}{n}}+\sqrt{\frac{\log N_{\mathcal{G}}}{n}}\right) $$
		with probability at least $1-2\exp (-2\log N_{\mathcal{G}})$. 	Thus it suffices to choose $\lambda\geq 4\sigma (\sqrt{\frac{m}{n}}+\sqrt{\frac{\log N_{\mathcal{G}}}{n}})$, as suggested in the corollary.
		
		Finally, the tolerance can be computed as
		$$\delta=c_3 s_{\mathcal{G}}\sigma_2(\Sigma)(\sqrt{\frac{m}{n}}+\sqrt{\frac{3 \log N_{\mathcal{G}}}{n}})^2 \|\hat{w}-w^*\|_2^2.$$
	\end{proof}

	\begin{proof}[Proof of corollary on Linear regression with SCAD regularizer, i.e., corollary \ref{cor.scad}]
		The proof is similar to that of Lasso. Recall in the proof of the Lasso example we have 
		$$ \|\nabla f(w^*)\|_\infty=\frac{1}{n}\|X^T\xi \|_\infty\leq 3 \sigma \sqrt{\frac{\log p}{n}},$$ 
		and the RSC condition is 
		$$ \frac{\|X\Delta\|_2^2}{n}\geq \frac{1}{2}\|\Sigma^{1/2}\Delta\|_2^2-c_1\nu(\Sigma)\frac{\log p}{n} \|\Delta\|_1^2.$$
		Recall $\mu=\frac{1}{\zeta-1}$ and $L_d=1$, we establish the result. 
	\end{proof}

	\begin{proof}[Proof of corollary on corrected Lasso, i.e., corollary \ref{cor.corrected_lasso}]
		Notice
		$$ \|\nabla f(w^*)\|_\infty=\|\hat{\Gamma}w^*-\hat{\gamma}\|_{\infty}=\|\hat{\gamma}-\Sigma w^* +(\Sigma-\hat{\Gamma})w^*\|_{\infty} \leq \|\hat{\gamma}-\Sigma w^*\|_\infty+\|(\Sigma-\hat{\Gamma})w^*\|_{\infty}.$$
		As shown in Lemma 2 of \cite{loh2011high}, both terms are bounded by
		$c_1 \varphi \sqrt{\frac{\log p}{n}}$ with probability at least $1-c_1\exp(-c_2\log p )$., where $\varphi=(\sqrt{\sigma_{\max} (\Sigma)}+\sqrt{\gamma_\varsigma}) (\sigma+\sqrt{\gamma_\varsigma} \|w^*\|_2).$

		To obtain the RSC condition,  we apply Lemma 1 in \cite{loh2011high}, to get
		$$ \frac{1}{n}\Delta^T\hat{\Gamma}\Delta \geq \frac{\sigma_{\min} (\Sigma)}{2} \|\Delta\|_2^2-c_3 \sigma_{\min}(\Sigma)\max \left( (\frac{\sigma_{\max} (\Sigma)+\gamma_\varsigma}{\sigma_{\min}(\Sigma)})^2,1 \right) \frac{\log p}{n} \|\Delta\|_1^2,$$
		with probability at least $1-c_4 \exp \left(-c_5 n\min \big( \frac{\sigma^2_{\min} (\Sigma)}{( \sigma_{\max}(\Sigma)+\gamma_\varsigma)^2},1 \big)    \right)$.
		
		Combine these pieces together, we establish the result.
	\end{proof}

	\bibliography{SDCA}
	\bibliographystyle{plainnat}
\end{document}